\documentclass{article} % For LaTeX2e
\usepackage{times}
\usepackage{eso-pic} % used by \AddToShipoutPicture
\usepackage[margin=1in]{geometry}
\usepackage[square,numbers,sort&compress]{natbib}
\usepackage{fancyhdr}

\usepackage{microtype}
\usepackage{graphicx}
\usepackage{subfig}
\usepackage{booktabs} % for professional tables
\usepackage{multirow}
\usepackage{multicol}
\usepackage{colortbl}
\usepackage{enumitem}
\usepackage{algorithm}
\usepackage{algorithmic}
\usepackage{hyperref}

\newcommand{\alglinelabel}{%
  \addtocounter{ALC@line}{-1}% Reduce line counter by 1
  \refstepcounter{ALC@line}% Increment line counter with reference capability
  \label% Regular \label
}

\usepackage{amsmath}
\usepackage{amssymb}
\usepackage{mathtools}
\usepackage{amsthm}
\usepackage[capitalize,noabbrev]{cleveref}

\theoremstyle{plain}
\newtheorem{theorem}{Theorem}
\newtheorem{proposition}[theorem]{Proposition}

\theoremstyle{definition}

\theoremstyle{remark}

\usepackage[textsize=tiny]{todonotes}

\title{Early Time Classification with Accumulated Accuracy Gap Control}

\author{} 
\usepackage{authblk}

\author[1]{Liran Ringel} 
\author[2]{Regev Cohen}
\author[2]{Daniel Freedman}
\author[2]{Michael Elad}
\author[1,3]{Yaniv Romano}

\affil[1]{Department of Computer Science, Technion IIT, Haifa, Israel}
\affil[2]{Verily AI, Israel}
\affil[3]{Department of Electrical and Computer Engineering, Technion IIT, Haifa, Israel}

\date{}

\begin{document}
\maketitle

\begin{abstract}
Early time classification algorithms aim to label a stream of features without processing the full input stream, while maintaining accuracy comparable to that achieved by applying the classifier to the entire input.
In this paper, we introduce a statistical framework that can be applied to any sequential classifier, formulating a calibrated stopping rule. This data-driven rule attains finite-sample, distribution-free control of the accuracy gap between full and early-time classification. We start by presenting a novel method that builds on the Learn-then-Test calibration framework to control this gap marginally, on average over i.i.d. instances.
As this algorithm tends to yield an excessively high accuracy gap for early halt times, our main contribution is the proposal of a framework that controls a stronger notion of error, where the accuracy gap is controlled conditionally on the accumulated halt times.
Numerical experiments demonstrate the effectiveness, applicability, and usefulness of our method. We show that our proposed early stopping mechanism reduces up to 94\% of timesteps used for classification while achieving rigorous accuracy gap control.
\end{abstract}

\section{Introduction}

\begin{figure}[h]
\centering
\includegraphics[width=0.9\textwidth]{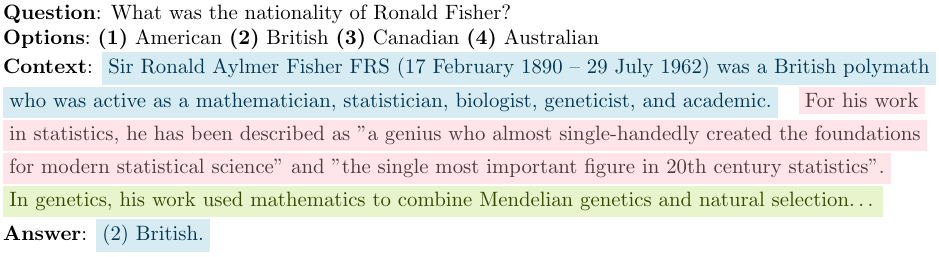}
% \vskip -0.1in
\caption{\textbf{An illustration of a reading comprehension task.} An LLM sequentially processes the given document to find the answer to the question provided and, ideally, should stop scanning the document immediately after the required information is found. The context is taken from \href{https://en.wikipedia.org/wiki/Ronald_Fisher}{Wikipedia}.}
\label{fig:ronald_fisher_example}
% \vskip -0.1in
\end{figure}

\noindent
The goal of early time series classification (ETSC) is to predict the label of a given input data stream as quickly as possible. Such methods are especially advantageous in scenarios requiring prompt predictive inference. 
For example, consider the problem of reading comprehension, illustrated in Figure~\ref{fig:ronald_fisher_example}. Suppose we employ an autoregressive large language model (LLM) to analyze a given document (context) and select an answer to the provided question. Given that the inference time of LLMs increases with the number of processed tokens (or sentences), we wish to terminate the processing of the context retrieved from the document as soon as the necessary information is found, rather than processing the entire document naively.
% A demonstration of this approach is shown in Figure \ref{fig:ronald_fisher_example}.
Other tasks for which ETSC is highly desired include real-time song identification (think of the \texttt{Shazam} application) and reducing radiation exposure in computational tomography (CT) systems, among many others.
% natural language processing applications that rely on resource-intensive large language models (LLMs), 
In all of these applications, the objective is to stop the inference process early while preserving accuracy, as if the predictive model had been applied to the entire data stream.

Consider labeled pairs of the form $(X, Y)$ sampled i.i.d. from $P_{XY}$, where $X = (X^1, X^2, \ldots, X^{t_{\text{max}}}) \in \mathcal{X}$ represents an observed input sequence with a maximum length of $t_{\text{max}}$, e.g. a sequence of tokens representing sentences in a document and a question. The variable $Y \in \mathcal{Y} = \{1,\ldots,K\}$ is the unknown label we wish to predict, e.g. the correct answer to the given question.
Suppose we are handed a pre-trained classifier $\hat{f}: \mathcal{X} \rightarrow [0,1]^K$ that processes the input $X$ sequentially and, at each timestep $t$, maps $X^{\leq t} = (X^1, \ldots, X^t)$ to an estimated probability distribution over the labels.
We employ a stopping rule function that, at each timestep $t$, decides whether to stop the inference process \emph{only based on the data observed up to timestep $t$}. Denote the stopping time by $\hat{\tau}(X)\in\{1,\dots,t_{\text{max}}\}$ and let $\hat{Y}_{\text{early}}(\hat{\tau})$ and $\hat{Y}_{\text{full}}$ be the predicted labels obtained by $\hat{f}(X^{\leq \hat{\tau}(X)})$ and $\hat{f}(X)$, respectively.
With these notations in place, we define the \textit{accuracy gap} as the proportion of samples for which the classifier's prediction is correct when applied to the entire sequence but incorrect when the same classifier is applied only up to the early timestep $\hat{\tau}(X)$.

Let $\alpha \in (0,1)$, e.g., 10\%, be the tolerable accuracy gap,  representing the acceptable trade-off for early stopping. Denote by $\mathcal{D}_{\text{cal}} = \{ (X_i, Y_i) \}_{i=1}^{n_{\text{cal}}}$ a \emph{holdout calibration set}, with samples drawn i.i.d. from $P_{XY}$. Our initial objective is to leverage $\mathcal{D}_{\text{cal}}$ to identify an early stopping rule $\hat{\tau}(X)$ that minimizes the halt time while ensuring the accuracy gap remains below $\alpha$ with a probability of at least $1-\delta$:
\begin{equation}
\label{eq:marginal_guarantee}
\mathbb{P}_{\mathcal{D}_{\text{cal}}} \left ( R^{\text{marginal}}_{\text{gap}}(\hat{\tau}) \leq \alpha \right ) \geq 1-\delta,
\end{equation}
where
\begin{equation}
\label{eq:R_marginal}
  R_{\text{gap}}^{\text{marginal}}(\hat{\tau}) = \mathbb{E}_{P_{XY}} \left [ L_{\text{gap}}(Y,\hat{Y}^{\text{full}},\hat{Y}^{\text{early}}(\hat{\tau})) \right ],
\end{equation}
and
\begin{equation}
  \label{eq:accuracy_gap}
  L_{\text{gap}}(Y,\hat{Y}^{\text{full}},\hat{Y}^{\text{early}}(\hat{\tau})) = \left(\mathbb{I}_{Y = \hat{Y}^{\text{full}}} - \mathbb{I}_{Y = \hat{Y}^{\text{early}}(\hat{\tau})} \right)_+.
\end{equation}
Notably, the probability in \eqref{eq:marginal_guarantee} is taken over the randomness in $\mathcal{D}_{\text{cal}}$, and $\delta$ is a user-defined level, e.g., 1\%. The operator $(z)_+$ in \eqref{eq:accuracy_gap} returns the value $z$ if $z\geq0$ and zero otherwise, and the indicator function $\mathbb{I}_{a = b}$ equals 1 when $a = b$ and zero otherwise. In simpler terms, the expected value of $L_{\text{gap}}(Y,\hat{Y}^{\text{full}},\hat{Y}^{\text{early}}(\hat{\tau}))\in \{0,1\}$ reflects the proportion of samples in which the decision to stop early increases the error rate. We refer to~\eqref{eq:marginal_guarantee} as \emph{marginal risk control} as it states that the accuracy gap will not exceed $\alpha$, \emph{on average} over future observations and stopping times. In Section~\ref{sec:marginal_risk_control}, we present an algorithm that rigorously attains~\eqref{eq:marginal_guarantee}, termed the \emph{marginal method}.

While the marginal guarantee in \eqref{eq:marginal_guarantee} provides a controlled mechanism for early classification, it may not be entirely satisfying in most practical settings. This is because an algorithm that controls the accuracy gap over all possible sequences is permitted to perform poorly on sequences with early halt times while excelling on sequences with late halt times. This, in turn, can undermine the reliability of predictions for sequences with early halt times. Recognizing this limitation, our main contribution is a novel algorithm that aims to \emph{control the accuracy gap conditional on the halt time being less or equal to $t$}. More formally, let
\begin{equation}
\label{eq:R_leq_t}
    R_{\text{gap}}^{\leq t}(\hat{\tau}) = \mathbb{E}_{P_{XY}} \left [ L_{\text{gap}}(Y,\hat{Y}^{\text{full}},\hat{Y}^{\text{early}}(\hat{\tau})) \mid \hat{\tau}(X) \leq t \right ].
\end{equation}
Our goal is to formulate a stopping rule that achieves
\begin{equation}
\label{eq:conditional_guarantee}
\begin{aligned}
\mathbb{P}_{\mathcal{D}_{\text{cal}}} \left ( R_{\text{gap}}^{\leq t}(\hat{\tau}) \leq \alpha \ \text{for all} \ t\geq t_0 \right ) \geq 1-\delta,
\end{aligned}
\end{equation}
where $t_0$ is defined as the first timestep for which $P(\hat{\tau}(X) \leq t_0) > 0$, as otherwise \eqref{eq:R_leq_t} is undefined.
In particular, controlling \eqref{eq:conditional_guarantee} implies that we also control the accuracy gap marginally, as $R_{\text{gap}}^{\text{marginal}} = R_{\text{gap}}^{\leq t_{\text{max}}}$. Throughout this work, we refer to \eqref{eq:conditional_guarantee} as \emph{conditional risk control on the accumulated halt time}, or simply, \emph{conditional risk control}. In Section \ref{sec:conditional_risk} we present an algorithm that achieves this goal, which we refer to as the \emph{conditional method}.

It is crucial to distinguish \eqref{eq:conditional_guarantee} from the stronger time- or instance-conditional guarantee, where the objective is to control the accuracy gap for a specific timestep $t$ or for a specific $X$. Unfortunately, attainment of non-trivial stopping rules with time- or instance-conditional risk control is infeasible without resorting to unrealistic assumptions \cite{vovk2012conditional,lei2014distribution,foygel2021limits}, which we aim to avoid: we pursue distribution-free, finite-sample guarantees. As a consequence, we posit that the risk in \eqref{eq:conditional_guarantee} strikes a reasonable compromise between controlling the relatively weak marginal risk and the unattainable time- or instance-conditional risk.

\subsection{A Motivating Example: Reading Comprehension}
\label{sec:teaser}

\begin{figure}[ht]
\centering
        \includegraphics[width=0.48\textwidth]{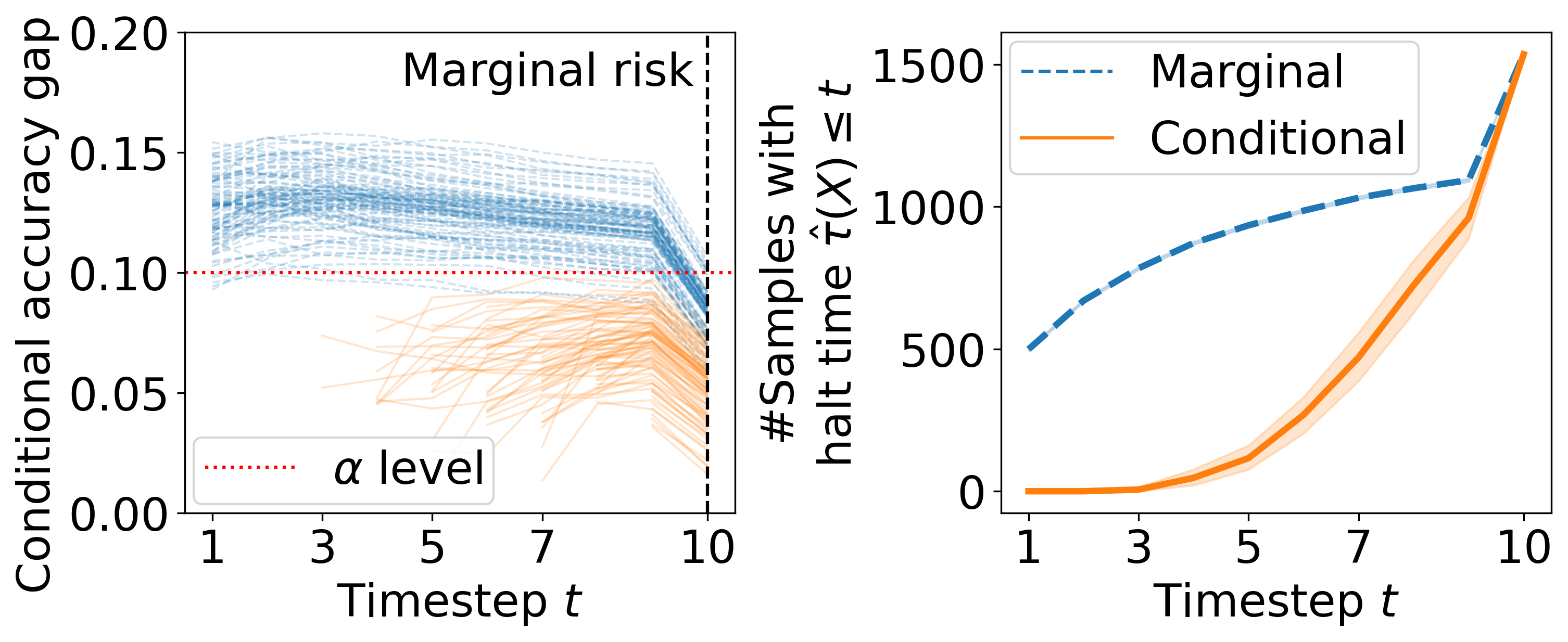}
        \vskip -0.1in
        \caption{\textbf{Comparison between the marginal and conditional methods for the reading comprehension task.} Nominal accuracy gap level is $\alpha=10\%$ and $\delta=1\%$. Left: empirical conditional accuracy gap, $\hat{R}_{\text{gap}}^{\leq t}$, across 100 trials; each curve corresponds to a different random split of the calibration and test data. Right: accumulated halt times as a function of $t$, averaged over 100 random splits; the shaded area represents a 95\% confidence interval.}
        \label{fig:quality_conditional_vs_marginal}
\end{figure}

To emphasize the importance of the transition from the marginal \eqref{eq:marginal_guarantee} to the conditional guarantee \eqref{eq:conditional_guarantee}, we now return to the reading comprehension problem discussed earlier. The \texttt{QuALITY} dataset \cite{pang2022quality} consists of 4609 triplets, containing (i) a question, (ii) multiple choice answers, and (iii) a long context, with each triplet accompanied by the correct labeled answer.
We utilize a pre-trained autoregressive LLM as the base predictive model. This classifier sequentially processes the context and selects an answer from the four possibilities. For the calibration of the early stopping rule, we employ 3073 labeled samples to form $\mathcal{D}_{\text{cal}}$ while reserving the remaining 1536 samples for testing. Following this, we compare the performance of the proposed marginal and conditional calibration methods presented in Sections \ref{sec:marginal_risk_control} and \ref{sec:conditional_risk}, respectively. Specifically, we report two performance metrics: (i) $\hat{R}_{\text{gap}}^{\leq t}$, defined as the empirical accuracy gap of samples with a halt time $\hat{\tau}(X)$ equal to or less than $t$; and (ii) the cumulative number of samples on which the model halted until timestep $t$. 

The results are presented in Figure \ref{fig:quality_conditional_vs_marginal}. Following the left panel in that figure, we can see that while the two approaches control the marginal risk, the conditional accuracy gap $\hat{R}_{\text{gap}}^{\leq t}$ tends to be higher than the desired $10\%$ level for the marginal method.
This implies that the marginal stopping rule tends to halt too early, as evidenced in the right panel of Figure~\ref{fig:quality_conditional_vs_marginal}, where we can see the relatively large number of samples halted at timestep $t=1$.
In contrast, the conditional approach maintains the conditional accuracy gap below $\alpha$ across all timesteps (left panel) while attaining an effective early stopping mechanism (right panel).

\subsection{Preview of our methods}
\label{sec:preview_of_our_methods}

The crux of this work is the formulation of a stopping rule $\hat{\tau}(X)$ that attains valid risk control.
Denote by $\hat{\pi}: \mathcal{X} \to [0,1]$ a score that heuristically reflects how confident the classifier is in its prediction based on $X^{\leq t}$.
For example, $\hat{\pi}(X^{\leq t})$ can be the largest softmax value of a neural net classifier.
With this in place, we can formulate 
\begin{equation}
\hat{\tau}(X) = \tau_{\hat{\underline{\lambda}}}(X) = \min \{ t : \hat{\pi}(X^{\leq t}) \geq \underline{\hat{\lambda}}_t \ \text{or} \ t=t_\text{max}\}, 
\end{equation}
where $\hat{\underline{\lambda}}_t$ is a hyperparameter, being the $t$-th element in a vector of thresholds $\underline{\hat{\lambda}} = (\underline{\hat{\lambda}}_1, \underline{\hat{\lambda}}_2,\dots,\underline{\hat{\lambda}}_{t_{\text{max}}})$.
In plain words, we choose to halt the inference process for the first time $t$ that the classifier is ``confident enough'' in its prediction.
But how can we properly choose the vector of hyperparameters $\underline{\hat{\lambda}}$ that attains a valid risk control?
Notably, this task becomes particularly challenging when dealing with a large number of hyperparameters that require tuning; in our case, we have $t_{\text{max}}$ parameters.
An improper choice of hyperparameters can fail to achieve the desired accuracy gap on future test data, and this problem is especially pronounced when the accuracy gap is a non-monotone function of the hyperparameters, which may occur in our setting due to the complex nature of the pre-trained classifier at hand.

To tackle this challenge, we build on the \emph{Learn then Test} (LTT) framework \cite{angelopoulos2021learn} that formulates the problem of finding hyperparameters that yield risk control as a multiple hypothesis testing problem, where each hypothesis corresponds to a different choice of hyperparameters.
However, in situations with a vast array of parameters that need to be tuned, this method faces two practical obstacles \citep{laufer2022efficiently}.
First, the sheer volume of potential configurations, which grows exponentially with $t_{\text{max}}$, makes an extensive search of hyperparameters infeasible.
Second, the LTT method may experience a loss of power when confronted with such an exponential number of tests. This drawback can result in our algorithm stopping too late, potentially missing the opportunity to select a more refined set of hyperparameters for the downstream task.
 
To alleviate these limitations, we propose a two-stage calibration framework that exploits the special structure of the underlying ETSC problem. In the first stage, we find a candidate set of hyperparameters using a novel computationally efficient procedure. Then, we apply a multiple testing procedure on the candidate set to select a valid set of hyperparameters that yields risk control. Overall, the novel algorithm we introduce can efficiently handle long sequences, while selecting a data-adaptive threshold vector $\hat{\underline{\lambda}}$ that formulates a statistically valid early stopping rule. In turn, the contributions of this work are the following:

\begin{enumerate}[leftmargin=*,topsep=0pt]
    \item \textbf{A novel application for LTT}: we introduce, for the first time, methodologies that support ETSC algorithms with rigorous distribution-free, finite-sample risk-controlling guarantees.
    \item \textbf{Marginal risk control}: we present a flexible framework that allows predictive models to stop early the inference process while controlling the average accuracy gap.
    \item \textbf{Conditional risk control}: next, we introduce a novel algorithm for early stopping that controls the accuracy gap conditional on the accumulated halt times.
    \item \textbf{Theory precisely holds in practice}: we illustrate the effectiveness of our algorithms by applying them to diverse tasks. These include standard time series classification datasets and a novel application in natural language processing (NLP). Our methods controls the risk while saving up to 94\% of the timesteps available to make predictions. A software package implementing the proposed methods is publicly available on GitHub.\footnote{\url{https://github.com/liranringel/etc}}
\end{enumerate}

\section{Related Work}

There is active research in developing machine learning models for ETSC with stopping rules that aim to balance accuracy and early termination \cite{hartvigsen2019adaptive, gupta2020approaches, ghodrati2021frameexit, sabet2021temporal, tang2022temporal, hartvigsen2022stop, chen2022decoupled, shekhar2023benefit}.
While these tools are effective in practice, they often lack statistical assurance. Our proposal enriches this important line of research by introducing versatile tools, compatible with any state-of-the-art ETSC model, which rigorously control the accuracy gap, be it in a marginal or conditional sense.

Our proposal is closely related to calibrated predictive inference techniques, including conformal prediction, risk-controlling methods, and selective classification \cite{vovk2005algorithmic,papadopoulos2011reliable,lei2018distribution,tibshirani2019conformal,romano2020classification,bates2021distribution,angelopoulos2021gentle,gibbs2021adaptive,lin2022conformal,angelopoulos2022conformal,fisch2022calibrated,feldman2023achieving,lee2023t,cauchois2023robust,barber2023conformal}. Specifically, we expand the toolbox of risk-controlling tools, particularly when facing situations with high dimensional hyperparameter space.
The pioneering LTT work by \citet{angelopoulos2021learn} offers an approach to find a data-driven configuration of parameters that, for example, can be used to simultaneously control multiple risks. However, this approach can mostly handle \emph{low dimensional hyperparameter space} and becomes intractable when the search space is large. Recognizing this limitation, \citet{laufer2022efficiently,laufer2023risk} utilize Bayesian optimization tools to find Pareto optimal candidate configurations across various risks, which, in turn, improve the computational and statistical efficiency of LTT. This line of work shares similarities with the challenges we face in this paper; however, instead of utilizing a general purpose Bayesian optimization tool for parameter tuning, or using exhaustive search, we design a specialized procedure that builds upon the structure of the ETSC problem. Our proposal results in a computationally efficient technique to identify plausible configurations among the potentially enormous search space of $\underline{\lambda}$, that controls the accuracy gap with meaningful statistical power. 

Our approach is also aligned with recent efforts to design calibration methods that aim to reduce the computational complexity of LLMs \cite{schuster2021consistent,schuster2022confident}.
These methods involve formulating an early exit mechanism with, for example, marginal accuracy gap control.
% in LLMs, usually involve attaching language modeling heads to intermediate transformer blocks.
% \citet{quach2023conformal} use LTT for early exit with marginal risk control.
A key difference between the above methods and our proposal is that we apply early exit over the time horizon rather than over the intermediate transformer layers. Furthermore, a crucial conceptual and technical difference is our transition from marginal to conditional guarantees, departing from the contributions mentioned above.

\section{Warm-up: Marginal Accuracy Gap Control}
\label{sec:marginal_risk_control}

To set the stage for our framework for conditional risk control, we start by presenting a method that achieves the modest marginal guarantee in \eqref{eq:marginal_guarantee}.
The development of this method also exposes the reader to the statistical principles of LTT.
Later, in Section \ref{sec:conditional_risk}, we will build on the foundations of the method presented here and introduce our main contribution---a methodology that attains the conditional guarantee of \eqref{eq:conditional_guarantee}.
To further simplify the exposition of the proposed marginal approach, consider tuning a single parameter $\hat{\lambda} \in [0,1] \cup \{\infty\}$ for all timesteps so that the stopping rule $\tau_{\hat{\lambda}}(X) = \min \{ t : \hat{\pi}(X^{\leq t}) \geq \hat{\lambda} \ \text{or} \ t=t_{\text{max}} \}$ achieves~\eqref{eq:marginal_guarantee}.

To start with, suppose we handed a candidate parameter $\lambda$, e.g., $\lambda=0.7$, and we are interested in testing whether it controls the accuracy gap. Following the LTT \cite{angelopoulos2021learn} approach, we define the null hypothesis induced by $\lambda$ as follows: 
\begin{equation}
  H_{0,\lambda}: R_{\text{gap}}^{\text{marginal}}(\tau_{\lambda}) > \alpha.
\end{equation}
That is, if the null is false, our candidate $\lambda$ controls the marginal accuracy gap. With this in place, we formulate a statistical test that utilizes the observed labeled data---the calibration set---to decide whether we can reject $H_{0,\lambda}$ and report that $\lambda$ is likely to control the risk or accept $H_{0,\lambda}$ if there is not enough evidence to reject the null. To formulate such a test, we compute a \emph{p-value} $p_\lambda$, where a valid p-value satisfies the following property under the null:
\begin{equation}
  \label{eq:p_value_property}
  \mathbb{P}_{\mathcal{D}_{\text{cal}}} \left ( p_\lambda \leq u \mid H_{0,\lambda} \right ) \leq u, \ \ u \in [0,1].
\end{equation}
In plain words, if $H_{0,\lambda}$ is true, the p-value is stochastically greater than or equal to uniform distribution on $[0,1]$.
Hence, \emph{considering a single hypothesis}, when observing $p_\lambda \leq \delta$ we can safely reject $H_{0,\lambda}$, knowing that the probability of falsely rejecting the null (type I error) is at most $\delta$. 

To compute such a p-value, we leverage the fact that the loss $L_\text{gap}$ is binary, and thus we can employ the exact tail bound from \citet{bates2021distribution} (Appendix B); see also \citet{brown2001interval}.
In more detail, denote the cumulative distribution function of the binomial distribution by $\text{CDF}_{\text{bin}}(\hat{k};n,\alpha)$ where $\hat{k}$ is the number of successes, $n$ is the number of independent Bernoulli trials, and $\alpha$ is the probability of success. Thus, in our case, the p-value is $\hat{p}_\lambda = \text{CDF}_{\text{bin}}\left(n\hat{R}_{ \text{gap}}(\tau_\lambda); n, \alpha\right)$, where $\hat{R}_{\text{gap}}(\tau_\lambda)$ is the empirical accuracy gap obtained by the stopping rule $\tau_\lambda$, evaluated on $n=|\mathcal{D}_{\text{cal}}|$ i.i.d. samples. Put simply, this formula transforms the empirical risk, evaluated on the calibration set $\mathcal{D}_{\text{cal}}$, into a p-value that satisfies~\eqref{eq:p_value_property}.

Thus far we have discussed the problem of testing for a single hypothesis, i.e., testing whether a specific candidate $\lambda$ does not control the accuracy gap. However, naturally, the task of finding $\hat{\lambda}$ that promotes early stopping while controlling the risk involves \emph{testing for multiple hypotheses}: each hypothesis $H_{0,\lambda_i}$ corresponds to a different $\lambda_i \in \Lambda, \ 1 \leq i \leq |\Lambda|$, where $\Lambda = \{ 0, \Delta, 2\Delta, \dots, 1 \} = \{ \lambda_1, \lambda_2, \dots, \lambda_{|\Lambda|}\}$ is a discretized grid of possible values and $\Delta \in (0, 1)$ defines the resolution of the grid.

The challenge that arises is that we must test all hypotheses simultaneously. To clarify, a naive rejection rule $p_{\lambda_i} \leq \delta$ can lead to a high probability that some of the true null hypotheses are rejected by chance alone, and this probability increases with the number of true nulls that are tested \cite{miller2012simultaneous}.
To tackle this issue, we follow \citet{angelopoulos2021learn} and formulate a multiple testing procedure that controls the \emph{family-wise error rate} (FWER). Formally, let $V$ be the number of true nulls that are falsely rejected by the testing procedure, and define  $\text{FWER} = \mathbb{P}\left( V \geq 1 \right)$ 
as the probability of falsely rejecting at least one true null hypothesis.
Therefore, to control \eqref{eq:marginal_guarantee}, we should design a testing procedure that ensures the FWER does not exceed $\delta$.

To rigorously control the FWER, we adopt the \emph{fixed sequence testing} procedure \cite{bauer1991multiple} used in LTT, as follows.
First, we order the hypotheses from most plausible to least \emph{without looking at the calibration data}.
In our context, higher thresholds are more likely to control the risk in~\eqref{eq:marginal_guarantee}, and therefore we order the hypotheses from the largest $\lambda_{|\Lambda|}$ to the smallest $\lambda_{1}$. Then, we arrange the p-values according to this ordering and sequentially compare each p-value to the desired level $\delta$. This sequential testing procedure terminates the first time $j$ that $p_{\lambda_j} > \delta$, resulting in a set of valid thresholds $\mathcal{R} = \{\lambda_i : i > j\}\cup\{\infty\}$.
Importantly, any threshold in the set $\mathcal{R}$ is guaranteed to control~\eqref{eq:marginal_guarantee}, including the trivial choice for which $\lambda=\infty$. (When $\lambda=\infty$ the model will never stop early and thus trivially achieves zero accuracy gap.)  Since our goal is to formulate a rule that stops as early as possible, we set the final $\hat{\lambda}$ to be the smallest $\lambda$ among the rejected ones, i.e., $\hat{\lambda}=\lambda_{j+1}$, or $\hat{\lambda} = \infty$ if $p_{\lambda_{|\Lambda|}} > \delta$.
For ease of reference, this procedure is summarized in Algorithm~\ref{alg:marginal_risk_control}, presented in Appendix~\ref{app:marginal_algo}, and its validity is a direct consequence of using fixed sequence testing to control the FWER at level $\delta$.
\begin{proposition}
  \label{prop:marginal_risk_control}
  Assuming the calibration and test samples are i.i.d., with $\hat{\lambda}$ selected as outlined in Algorithm~\ref{alg:marginal_risk_control}, the stopping rule $\tau_{\hat{\lambda}}(X)$ satisfies \eqref{eq:marginal_guarantee}. 
\end{proposition}
All proofs are presented in Appendix~\ref{app:proofs}.
In plain words, the above proposition implies that Algorithm~\ref{alg:marginal_risk_control} formulates a stopping rule that achieves marginal accuracy gap control given a finite calibration set, no matter what the data distribution is, and regardless of the choice of the ``black-box'' classifier.
While Proposition~\ref{prop:marginal_risk_control} is appealing, the usefulness of the marginal guarantee in real-world scenarios may be limited, as discussed and demonstrated in Section~\ref{sec:teaser}. This limitation prompts our exploration in the next section.

\section{Conditional Accuracy Gap Control}
\label{sec:conditional_risk}
We now turn to present the focal point of this work: a framework designed to control the conditional accuracy gap~\eqref{eq:conditional_guarantee}. Beyond the transition from marginal to conditional guarantee, in this section we utilize a more general formulation of the stopping rule, in which $\hat{\tau}(X) = \tau_{\underline{\hat{\lambda}}}(X) = \min \{ t : \hat{\pi}(X^{\leq t}) \geq \underline{\hat{\lambda}}_t \}$ with $\underline{\hat{\lambda}} = (\underline{\hat{\lambda}}_1, \underline{\hat{\lambda}}_2, \dots, \underline{\hat{\lambda}}_{t_{\text{max}}})$. This choice adds additional flexibility to the proposed framework compared to tuning a single parameter (as in Section~\ref{sec:marginal_risk_control}), allowing us to formulate more effective early stopping rules.

Analogously to Section~\ref{sec:marginal_risk_control}, we will adopt the fixed sequence testing procedure to construct a rejection set $\mathcal{R}$ that contains the configurations of $\underline{\lambda}$ that control the conditional risk. In the view of multiple testing, now each null hypotheses is formulated as 
\begin{equation}
\label{eq:conditional_null_joint}
  H_{0,\underline{\lambda}}: R_{\text{gap}}^{\leq t}(\tau_{\underline{\lambda}}) > \alpha \ \ \text{for at least one } t \geq t_0,
\end{equation}
where $t_0$ is the first timestep at which the probability for an early stopping event is not zero, i.e., $P(\tau_{\underline{\lambda}}(X) \leq t_0) > 0$.

In striking contrast to Section~\ref{sec:marginal_risk_control}, the formulation of a FWER-controlling procedure in this case is far more challenging due to the following.
\begin{enumerate}[leftmargin=*,topsep=0pt]
    \item There are $(|\Lambda|+1)^{t_\text{max}}$ possible configurations for $\underline{\lambda}$ and thus it is infeasible to sweep over this exponential number of hypotheses. Given this sheer volume, computing a p-value for each hypothesis exceeds reasonable computational limits.
    \item To achieve good statistical power with fixed sequence testing, careful ordering of hypotheses is essential: inadequate ordering may lead to a rejection set $\mathcal{R}$ that includes less effective threshold vectors. As discussed in Section~\ref{sec:marginal_risk_control}, there is a natural ordering of the hypotheses when considering the tuning of a \emph{single} threshold; we can simply order the hypotheses from the largest $\lambda$ to the smallest one. However, it is unclear how to order the hypotheses when working with a vector $\underline{\lambda}$.
    \item When faced with a small sample size, the p-value may be too high even if the risk is lower than $\alpha$. This is attributed to the fact that the p-value produced by $\text{CDF}_{\text{bin}}$ takes into account the number of samples used to calculate the empirical risk. Importantly, this is not an abstract concern; in practice, as we strive for conditional risk control, situations with a small sample size become prevalent, particularly for the very early timesteps.
\end{enumerate}

$\\$

In what follows, we present a method that alleviates these issues, taking inspiration from the principle of \emph{split fixed sequence testing} proposed in LTT. In this approach, we first split the calibration set $\mathcal{D}_{\text{cal}}$ into two disjoint sets: $\mathcal{D}_{\text{cal-1}}$ and $\mathcal{D}_{\text{cal-2}}$. Then, we proceed with a two-stage algorithm, described below at a high level.
\paragraph{\textbf{Stage 1: Candidate Screening}:} Use $\mathcal{D}_{\text{cal-1}}$ to heuristically find a data-adaptive threshold vector $\hat{\underline{\eta}}$, with an eye towards early stopping with conditional risk control.
\paragraph{\textbf{Stage 2: Testing}:} Apply fixed sequence testing to configurations derived from $\hat{\underline{\eta}}$. Here, we use the independent holdout set $\mathcal{D}_{\text{cal-2}}$ to ensure the validity of the test.

\subsection{Stage 1: Candidate Screening}
\label{sec:candidate_screening}

\begin{algorithm}[ht]
  \caption{Candidate Screening (Stage 1)}
  \label{alg:candidate_screening}
  \begin{algorithmic}[1]
  \STATE \textbf{Input}: Calibration set $\mathcal{D}_{\text{cal-1}} = \{ (X_i, Y_i) \}_{i=1}^{n_{\text{cal-1}}}$, tolerable accuracy gap $\alpha$, grid resolution $\Delta$.
  \STATE \( \hat{\underline{\eta}} \leftarrow \{\infty, \ldots, \infty\} \)
  \STATE // Find $\hat{\underline{\eta}}_t$ greedily during the $t$-th iteration.
  \FOR{\( t = 1, \ldots, t_{\text{max}} \)}
    \STATE \( \underline{\eta} \leftarrow \underline{\hat{\eta}} \)
    \STATE // Find the lowest $\underline{\eta}_t \in \Lambda$ s.t. $\hat{R}_{\text{gap}}^{\leq t} \leq \alpha$.
    \FOR{\( \xi = 0, \Delta, 2\Delta, \ldots, 1 \)}
      \STATE \( \underline{\eta}_t \leftarrow \xi \)
      \STATE \( I \leftarrow \{i : \tau_{\underline{\eta}}(X_i) \leq t \} \) \ // Find samples with a halt time $\leq t$.
      \IF{\( I = \emptyset \)}
        \STATE // Cannot calculate the empirical risk.
        \STATE Break inner loop and set $\hat{\underline{\eta}}_t = \infty$
      \ENDIF
      \STATE \( \hat{R}_{\text{gap}}^{\leq t} \leftarrow \frac{1}{|I|}\sum_{i \in I} L_{\text{gap}}(Y_i,\hat{Y_i}^{\text{full}},\hat{Y_i}^{\text{early}}(\tau_{\underline{\eta}})) \) \ // Calculate the empirical risk. \alglinelabel{alg:candidate_screening:calc_r_hat}
      \IF{\( \hat{R}_{\text{gap}}^{\leq t} \leq \alpha \)}
        \STATE // Found the lowest $\underline{\eta}_t$ s.t. $\hat{R}_{\text{gap}}^{\leq t} \leq \alpha$.
        \STATE Break inner loop and set  $\hat{\underline{\eta}}_t \leftarrow \xi$ 
      \ENDIF
    \ENDFOR
  \ENDFOR
  \STATE \textbf{Output}: $\hat{\underline{\eta}}$
  \end{algorithmic}
\end{algorithm}

\noindent
We present a greedy algorithm that takes as input a predictive model and calibration data $\mathcal{D}_{\text{cal-1}}$ and returns a candidate threshold vector $\hat{\underline{\eta}}$. This procedure, summarized in Algorithm~\ref{alg:candidate_screening}, sequentially updates the elements in the vector $\underline{\hat{\eta}}$ as follows. It starts by updating the first element $\hat{\underline{\eta}}_1$ that corresponds to the timestep $t=1$, then proceeds to $\hat{\underline{\eta}}_2$ for $t=2$, and continues until reaching $\hat{\underline{\eta}}_{\text{max}}$ at $t = t_{\text{max}}$. Specifically, at timestep $t$, we are handed the vector $\underline{\hat{\eta}} = (\hat{\underline{\eta}}_1, \ldots, \hat{\underline{\eta}}_{t-1}, \infty,\ldots,\infty)$, and set its $t$-th element $\hat{\underline{\eta}}_t$ to be the smallest ${\underline{\eta}}_t$ such that $\hat{R}_{\text{gap}}^{\leq t}(\tau_{\hat{\underline{\eta}}}) \leq \alpha$, or keep $\hat{\underline{\eta}}_t=\infty$ if there is no $\underline{\eta}_t$ that satisfies this constraint.
Above, $\hat{R}_{\text{gap}}^{\leq t}(\tau_{\hat{\underline{\eta}}})$ is the empirical accuracy gap of the samples with halt time that is less than or equal to $t$ (see line \ref{alg:candidate_screening:calc_r_hat}).

Before moving to the next stage, we pause to discuss the properties of this greedy method. First, the computational complexity of the proposed algorithm is $\mathcal{O}(t_{\text{max}}\cdot |\Lambda| \cdot |\mathcal{D}_{\text{cal-1}}|)$, which is attributed to the fact that we choose to sequentially update the vector $\underline{\hat{\eta}}$.
Second, by design, the choice of $\hat{\underline{\eta}}_{t'}$ for $t'>t$ does not affect $\hat{R}_{\text{gap}}^{\leq t'}$ for $t'\leq t$.
Third, this greedy method seeks a vector $\underline{\hat{\eta}}$ that yields a stopping rule whose empirical conditional risk is tightly regulated around $\alpha$, but not exceeded. This property is crucial to attaining an effective early stopping rule. In principle, instead of determining $\hat{\underline{\eta}}_t$ solely based on empirical risk, we could choose the smallest $\underline{\eta}_t$ whose p-value falls below $\delta$, an approach that is akin to the split fixed sequence testing idea of \citet{angelopoulos2021learn}. However, we decided to work directly with the empirical risk, as it is arguably straightforward to implement, and we found these two approaches to have similar halt times. In any case, while sensible, the process of finding the vector $\underline{\hat{\eta}}$ is heuristic in the sense that it is not guaranteed to control the conditional risk for future test points. This issue naturally leads us to the next stage.

\subsection{Stage 2: Testing}
\label{sec:valid_stage}

\begin{algorithm}[t]
\caption{Testing (Stage 2)}
\label{alg:valid_stage}
\begin{algorithmic}[1]
\STATE \textbf{Input}: Calibration set $\mathcal{D}_{\text{cal-2}} = \{ (X_i, Y_i) \}_{i=1}^{n_{\text{cal-2}}}$, candidate thresholds $\hat{\underline{\eta}}$, tolerable accuracy gap $\alpha$, significance level $\delta$, grid resolution $\Delta$.
\STATE // Start with the most conservative stopping rule.
\STATE \( \hat{\underline{\lambda}} \leftarrow \{\infty, \ldots, \infty\} \)
\STATE // Gradually reveal another $\hat{\underline{\eta}}_t$ from the end and test it.
\FOR{\( t = t_{\text{max}}, \ldots, 1 \)}
  \STATE \( \underline{\lambda} \leftarrow \hat{\underline{\lambda}} \)
  \STATE \( \underline{\lambda}_t \leftarrow \hat{\underline{\eta}}_t \) \ // Set $\underline{\lambda}$ to $\underline{\lambda}^t$.
  \STATE // Test $H_{0,\underline{\lambda}^t}^{t'}$ for all $t'\geq t$.
  \FOR{\( t' = t, \ldots, t_{\text{max}} \)}
    \STATE \( I \leftarrow \{i : \tau_{\underline{\lambda}}(X_i) \leq t' \} \) \ // Find samples with a halt time $\leq t'$.
    \IF{\( I = \emptyset \)}
        \STATE // No evidence to reject the null, stop testing.
        \STATE Break both loops
    \ENDIF
    \STATE \( \hat{R}_{\text{gap}}^{\leq t'} \leftarrow \frac{1}{|I|}\sum_{i \in I} L_{\text{gap}}(Y_i,\hat{Y_i}^{\text{full}},\hat{Y_i}^{\text{early}}(\tau_{\underline{\lambda}})) \) \ // Calculate the empirical risk.
    \STATE \( \hat{p}_{\underline{\lambda}^t}^{t'} \leftarrow \text{CDF}_{\text{bin}}\left(\hat{R}_{\text{gap}}^{\leq t'} \cdot |I|; |I|, \alpha\right) \) \ // Compute a p-value.
    \IF{\( \hat{p}_{\underline{\lambda}^t}^{t'} > \delta \)}
      \STATE // Failed to reject the null, stop testing.
      \STATE Break both loops
    \ENDIF
  \ENDFOR
  \STATE \( \hat{\underline{\lambda}} \leftarrow \underline{\lambda} \) \hspace{0.1cm} // $H_{0,\underline{\lambda}^t}$ was rejected, update the chosen \( \hat{\underline{\lambda}} \).
\ENDFOR
\STATE \textbf{Output}: \( \hat{\underline{\lambda}} \)
\end{algorithmic}
\end{algorithm}

\noindent
In this testing stage, we build on the candidate vector $\underline{\hat{\eta}}$ to form a statistically valid stopping rule that attains \eqref{eq:conditional_guarantee}. A naive and optimistic approach would be to test for a single null $H_{0,\underline{\lambda}}$ defined in \eqref{eq:conditional_null_joint} for the choice $\underline{\lambda} = \underline{\hat{\eta}}$. Rejection of this null hypothesis with a significance level of $\delta$ implies that $\underline{\hat{\eta}}$ attains \eqref{eq:conditional_guarantee}, achieving a powerful stopping rule due to the design of $\underline{\hat{\eta}}$. However, if we fail to reject this null, our fallback is the trivial configuration $\underline{\hat{\lambda}} = (\infty, \ldots, \infty)$ that results in a conditional accuracy gap of zero. However, in this case, the stopping rule we form is the most conservative one, as the model will never stop early.

To alleviate this, we employ fixed sequence testing, designed to yield an effective stopping rule with FWER-control, even in cases where the null hypothesis $H_{0,\underline{\lambda}}$ with the ``optimistic'' configuration $\underline{\lambda}=\underline{\hat{\eta}}$ would not be rejected. Recall the underlying principle of fixed sequence testing: order the hypotheses from the most plausible to the least, without looking at the holdout data $\mathcal{D}_{\text{cal-2}}$. Building on the structure of the ETSC problem, we define the sequence of configurations
$\underline{\lambda}^{t_{\text{max}}} = (\infty, \ldots, \infty, \hat{\eta}_{t_{\text{max}}})$, $\underline{\lambda}^{t_{\text{max}}-1} = (\infty, \ldots, \infty,\hat{\eta}_{t_{\text{max}}-1}, \hat{\eta}_{t_{\text{max}}})$, all the way to $\underline{\lambda}^{1} = (\hat{\eta}_{1}, \hat{\eta}_{2}, \ldots \hat{\eta}_{t_{\text{max}}})$. That is, the $t'$-th element in the vector $\underline{\lambda}^{t}$ is $ \underline{\lambda}^t_{t'} = \hat{\eta}_{t'} \ \text{if } t' \geq t, $ and $ \underline{\lambda}^t_{t'} = \infty$ otherwise.
Importantly, the stopping rule $\tau_{\underline{\lambda}^t}$ does not allow stopping the classification process at timesteps smaller than $t$. With this construction in place, we suggest applying the fixed sequence testing procedure to the hypotheses ordered from the one induced by $\underline{\lambda}^{t_{\text{max}}}$, i.e., $H_{0,\underline{\lambda}^{t_{\text{max}}}}$ down to the one corresponding to $\underline{\lambda}^{1}$, i.e., $H_{0,\underline{\lambda}^{1}}$. Note that this ordering is particularly powerful when the accuracy gap of the model tends to decrease with the number of timesteps observed---a sensible characteristic in ETSC. Additionally, the suggested ordering enables us to postpone the testing of hypotheses involving limited sample sizes to later stages of the procedure, which is attractive as it is more likely that we will fail to reject those nulls.

Having defined the ordering of the hypotheses, we turn to describe how to compute a valid p-value for each of the individual hypotheses, using the holdout data $\mathcal{D}_{\text{cal-2}}$. Consider the hypothesis $H_{0,\underline{\lambda}}$ in \eqref{eq:conditional_null_joint} for the choice $\underline{\lambda}=\underline{\lambda}^{t}$, and define its finer null hypotheses as follows:
\begin{equation}
\label{eq:conditional_null_t}
H_{0,\underline{\lambda}^{t}}^{t'} : R_{\text{gap}}^{\leq t'}(\tau_{\underline{\lambda}^{t}}) > \alpha \ \ \text{for} \ \ t'=t,\dots,t_\text{max}.
\end{equation}
Observe that $H_{0,\underline{\lambda}^{t}}$ in \eqref{eq:conditional_null_joint} is true if and only if there exists $t' \geq t$ such that $H_{0,\underline{\lambda}^{t}}^{t'}$ is true. Observe also that, by construction, $\tau_{\underline{\lambda}^t}$ cannot stop at timesteps smaller than $t$, and thus $t_0$ in \eqref{eq:conditional_null_joint} satisfies $t_0 \geq t$.
Importantly, the formulation of the finer nulls in \eqref{eq:conditional_null_t} paves the way to test the individual hypothesis $H_{0,\underline{\lambda}^t}$. Specifically, it implies that we can reject the individual hypothesis $H_{0,\underline{\lambda}^t}$ if all the finer hypotheses $H_{0,\underline{\lambda}^t}^{t'}, \  t'\geq t$ are rejected. This amounts to computing a p-value $\hat{p}_{\underline{\lambda}^t}^{t'}$ for each finer hypothesis $H_{0,\underline{\lambda}^t}^{t'}$ and rejecting $H_{0,\underline{\lambda}^t}$ if $\hat{p}_{\underline{\lambda}^t}^{t'} \leq \delta$ for all $t'\geq t$. Put simply, we reject $H_{0,\underline{\lambda}^t}$ if $\hat{p}_{\underline{\lambda}^t} = \max \{\hat{p}_{\underline{\lambda}^t}^{t'} : {t' \geq t} \} \leq \delta$.

Algorithm~\ref{alg:valid_stage} summarizes the proposed testing procedure. The outer loop in this algorithm sequentially iterates over the hypotheses $H_{0,\underline{\lambda}^t}$ from $\underline{\lambda}^{t_\text{max}}$ to $\underline{\lambda}^{1}$. The inner loop tests the null $H_{0,\underline{\lambda}^t}$ under study by breaking it into the finer hypotheses $H_{0,\underline{\lambda}^t}^{t'}, \  t'\geq t$. This algorithm returns the configuration $\underline{\hat{\lambda}}={\underline{\lambda}}^t$ corresponding to the smallest $t$ in which $H_{0,\underline{\lambda}^t}$ was rejected. The complexity of Algorithm~\eqref{alg:valid_stage} is $\mathcal{O}(t_\text{max}^2 \cdot |\mathcal{D}_{\text{cal-2}} |)$, and the validity of the resulting stopping rule $\tau_{\underline{\hat{\lambda}}}$ is as follows.

\begin{proposition}
  \label{prop:conditional_risk_control}
  Assuming the calibration and test samples are i.i.d., with $\hat{\underline{\lambda}}$ selected as outlined in Algorithm~\ref{alg:valid_stage}, the stopping rule $\tau_{\hat{\underline{\lambda}}}(X)$ satisfies \eqref{eq:conditional_guarantee}.
\end{proposition}
Similarly to Proposition~\ref{prop:marginal_risk_control}, the above result states that Algorithm~\eqref{alg:valid_stage} achieves a finite-sample, distribution-free risk control. But, in contrast with Proposition~\ref{prop:marginal_risk_control}, here we control a stronger notion of error---the conditional accuracy gap.

\section{Experiments}
\label{sec:experiments}
In this section, we evaluate the proposed methods both on structured time series datasets that are widely used in the ETSC literature and on the multiple-choice answering task, which was introduced in Section~\ref{sec:teaser}. 
The performance metrics include the conditional $R_{\text{gap}}^{\leq t}(\hat{\tau})$ and marginal $R_{\text{gap}}^{\text{marginal}}(\hat{\tau})=R_{\text{gap}}^{\leq t_\text{max}}(\hat{\tau})$ accuracy gap, evaluated on unseen test data $\mathcal{D}_{\text{test}}$. We also report the gain in early stopping, defined as the average normalized halt time:
$T_{\text{avg}} = \frac{1}{|\mathcal{D}_{\text{test}}|}\sum_{X_i\in\mathcal{D}_{\text{test}}} \frac{\hat{\tau}(X_i)}{t_\text{max}}.$ In all experiments, we set the target accuracy gap level to $\alpha = 10\%$, with $\delta=1\%$ and $\Delta=0.01$.
Throughout this section, the marginal method can be thought of as a baseline, as it closely resembles the calibration procedure suggested by \citet{schuster2022confident} to control the accuracy gap for early exit in transformers.

\subsection{Application to Structured Data}
\label{sec:application_to_structured_data}

\definecolor{grayish}{rgb}{0.95, 0.95, 0.95}
\begin{table}[ht]
\caption{\textbf{Summary of performance metrics for the proposed marginal and conditional methods across all structured datasets}. Results are presented for a nominal accuracy gap of $\alpha=10\%$ and $\delta=1\%$. The table provides the accumulated accuracy gap over the 20\% and 50\% earliest stopping times determined by $\hat{\tau}$ for each method, along with the marginal accuracy gap. The rightmost column presents the average normalized stopping time. All performance metrics are averaged over 100 random calibration/test splits. All standard errors are less than 0.008 and thus omitted.}
\label{tab:results}
\centering
\scalebox{0.9}{
\begin{tabular}{lclcccccc}
\toprule
\multirow{2}{*}{\textbf{Dataset}} & \multirow{2}{*}{\textbf{Late Acc.}} & \multirow{2}{*}{\textbf{Method}} & \multirow{2}{*}{\textbf{Early Acc.}} & \multicolumn{3}{c}{\textbf{Acc. Gap for Earliest $\hat{\tau}(X)$}} & \multirow{2}{*}{$T_{\text{avg}}$} \\ \cline{5-7}
& & & & 20\% earliest & 50\% earliest & Marginal & \\
\midrule
& & Marginal & 0.757 & 0.081 & 0.084 & 0.093 & 0.209 \\
& & Conditional & 0.771 & 0.064 & 0.065 & 0.085 & 0.215 \\
\rowcolor{grayish}
\multirow{-4}{*}{\texttt{Tiselac}} & \multirow{-4}{*}{0.816}
& Marginal & 0.809 & \color{red}0.117 & \color{red}0.108 & 0.079 & 0.471 \\
\rowcolor{grayish}
& & Conditional & 0.825 & 0.030 & 0.031 & 0.075 & 0.552 \\
\multirow{-4}{*}{\texttt{ElectricDevices}} & \multirow{-4}{*}{0.873}
& Marginal & 0.912 & 0.086 & 0.090 & 0.080 & 0.446 \\
& & Conditional & 0.940 & 0.049 & 0.050 & 0.051 & 0.567 \\
\rowcolor{grayish}
\multirow{-4}{*}{\texttt{PenDigits}} & \multirow{-4}{*}{0.989}
& Marginal & 0.608 & \color{red}0.171 & \color{red}0.135 & 0.086 & 0.580 \\
\rowcolor{grayish}
& & Conditional & 0.642 & 0.057 & 0.063 & 0.079 & 0.450 \\
\multirow{-4}{*}{\texttt{Crop}} & \multirow{-4}{*}{0.673}
& Marginal & 0.884 & 0.004 & 0.054 & 0.079 & 0.125 \\
& & Conditional & 0.901 & 0.033 & 0.039 & 0.067 & 0.061 \\
\bottomrule
\multirow{-4}{*}{\texttt{WalkingSittingStanding}} & \multirow{-4}{*}{0.962}
\end{tabular}}
\end{table}

\noindent
In this subsection, we test the applicability of our methods on five datasets:
\texttt{Tiselac} \cite{tiselac2017}, \texttt{ElectricDevices} \cite{UCRArchive}, \texttt{PenDigits} \cite{misc_pen-based_recognition_of_handwritten_digits_81}, \texttt{Crop} \cite{tan2017indexing}, and \texttt{WalkingSittingStanding} \cite{misc_human_activity_recognition_using_smartphones_240}.
These datasets are publicly available via the \href{https://www.aeon-toolkit.org/}{aeon} toolkit.
We refer to these as structured datasets as $X \in \mathbb{R}^{t_{\text{max}} \times d}$ and $X^t \in \mathbb{R}^d$. See Table~\ref{tab:structured_datasets} in Appendix~\ref{sec:structure_datasets_details} for more details.

To implement and evaluate our methods, we partition each dataset into four distinct sets: 80\% of the samples are allocated for model fitting, while the remaining samples are equally divided to form $\mathcal{D}_{\text{cal-1}}$, $\mathcal{D}_{\text{cal-2}}$, and $\mathcal{D}_{\text{test}}$. For the marginal method, we set $\mathcal{D}_{\text{cal}} = \mathcal{D}_{\text{cal-1}} \cup \mathcal{D}_{\text{cal-2}}$. In all experiments, we employ an LSTM model \cite{hochreiter1997long} as the base sequential classifier. A detailed description of the model architecture and training strategy is provided in Appendix~\ref{app:structure_datasets_model}.

The results obtained by the marginal and conditional methods are summarized in Table~\ref{tab:results}; see Appendix~\ref{sec:structure_datasets_results} for more detailed results for each dataset.
% We summarize the results in Table~\ref{tab:results}, which includes the accuracy gap for the first 20\%, 50\%, and 100\% of samples with the earliest halt time. Additionally, we provide the average portion of timesteps utilized by the model across 100 different seeds.
Following this table, the two methods control the marginal accuracy gap, supporting our theory. However, the marginal method fails to control the conditional risk for sequences with early halt times, in contrast with the conditional approach that attains valid risk control over the accumulated halt times---as guaranteed by our theory. The statistical efficiency of both methods is comparable, as evidenced by the average normalized halt time $T_\text{avg}$ performance metric.
In fact, although the conditional method controls a stronger notion of error, it resulted in a smaller average normalized stopping time $T_{\text{avg}}$ in 2 out of 5 datasets.
We attribute this gain to our decision to employ a vector of thresholds to form the conditional stopping rule, as opposed to the single threshold used in the baseline marginal approach.
Lastly, Figure~\ref{fig:t_avg_vs_alpha} in Appendix~\ref{sec:structure_datasets_results} illustrates the trade-off between the tolerable accuracy gap $\alpha$ and the average stopping time for the \texttt{Tiselac} dataset. There, one can see that the conditional method allows for earlier stopping times when a higher accuracy gap is permitted.

\subsection{An NLP Application}
\label{sec:nlp_application}
We now revisit the reading comprehension task introduced in Section~\ref{sec:teaser}, where the goal is to select the correct answer from a set of four options based on a given context.
To allow the sequential processing of the data, we first divide the context of each question into sentences. These sentences are then grouped into $t_\text{max}=10$ sets. When the total number of sentences cannot be grouped into 10 equally sized sets, we include  the remaining sentences in the last set.
To formulate the input sequence $X^{\leq t}$, we construct a prompt that includes the context sentences up to timestep $t$, along with the question and its four options, labeled `A', `B', `C', and `D'.
The prompt concludes with ``\texttt{The answer is:\textbackslash n\textbackslash n}'', which is then fed to the Vicuna-13B model \cite{zheng2023judging}
to make a prediction; the model is accessible via \href{https://huggingface.co/lmsys/vicuna-13b-v1.5-16k}{HuggingFace}. We employ the vLLM framework \cite{kwon2023efficient} to compute the probability assigned to each of the four options, resulting in $\hat{f}(X^{\leq t}) \in [0,1]^4$.
Lastly, we define the function $\hat{\pi}(X^{\leq t}) = \max \{ \hat{f}_k(X^{\leq t}) : k=1,\ldots,4 \}$, which is utilized to formulate the stopping rule $\hat{\tau}$.

The results obtained by the marginal and conditional methods are presented in Figure~\ref{fig:quality_conditional_vs_marginal}. As portrayed in the left panel, the conditional approach rigorously controls the conditional accuracy gap on the accumulated halt times, in contrast with the marginal method that merely controls the marginal risk. The right panel in Figure~\ref{fig:quality_conditional_vs_marginal} shows that the marginal method tends to stop earlier.
This is also indicated by its lower average normalized halt time of 0.483 compared to 0.831 for the conditional method.
However, this gain is not necessarily desired, as the marginal approach tends to make errors in the early halt times.

Figure~\ref{fig:stage1_vs_stage2} in the Appendix presents an ablation study, underscoring the importance of the testing phase (Stage 2) of the conditional method. As illustrated, the candidate configuration $\hat{\underline{\eta}}$ obtained by the greedy candidate screening algorithm (Stage 1) does not provide rigorous control of the conditional accuracy gap in the sense of \eqref{eq:conditional_guarantee}. This stands in contrast with the conditional method that includes the testing stage. Nevertheless, the candidate $\hat{\underline{\eta}}$ provides a reasonable initial set of configurations for the hyperparameters to be tested, as it yields a stopping rule that roughly centers around the nominal accuracy gap level $\alpha$.

\section{Conclusion}
\label{sec:conclusion}

In this paper, we presented a novel statistical framework that rigorously controls the accuracy gap conditional on the accumulated halt times.
Additionally, we performed a series of numerical experiments that highlight the significance of transitioning from marginal to conditional guarantees, which validates our theory and underscores the practical implications of our proposal.

Our work opens several future research directions.
First, it would be intriguing to design more effective stopping rules at the cost of increasing the computational complexity of the proposed two-stage calibration procedure.
Another direction is to address the limitation of our work---the reliance on the i.i.d. assumption, which may be violated in practice. It will be illuminating to extend the tools we presented and relax this assumption, possibly by relying on the foundations of \citet{barber2023conformal}.
Lastly, a natural progression is to extend the tools we developed to regression problems \cite{ye2023learned}. The challenge here is that the loss $L_{\text{gap}}$ might not be binary, and thus the exact p-value we used in this paper would not be applicable.

\bibliography{main}

\begin{thebibliography}{45}
\providecommand{\natexlab}[1]{#1}
\providecommand{\url}[1]{\texttt{#1}}
\expandafter\ifx\csname urlstyle\endcsname\relax
  \providecommand{\doi}[1]{doi: #1}\else
  \providecommand{\doi}{doi: \begingroup \urlstyle{rm}\Url}\fi

\bibitem[Vovk(2012)]{vovk2012conditional}
Vladimir Vovk.
\newblock Conditional validity of inductive conformal predictors.
\newblock In \emph{Asian conference on machine learning}, pages 475--490. PMLR, 2012.

\bibitem[Lei and Wasserman(2014)]{lei2014distribution}
Jing Lei and Larry Wasserman.
\newblock Distribution-free prediction bands for non-parametric regression.
\newblock \emph{Journal of the Royal Statistical Society Series B: Statistical Methodology}, 76\penalty0 (1):\penalty0 71--96, 2014.

\bibitem[Foygel~Barber et~al.(2021)Foygel~Barber, Candes, Ramdas, and Tibshirani]{foygel2021limits}
Rina Foygel~Barber, Emmanuel~J Candes, Aaditya Ramdas, and Ryan~J Tibshirani.
\newblock The limits of distribution-free conditional predictive inference.
\newblock \emph{Information and Inference: A Journal of the IMA}, 10\penalty0 (2):\penalty0 455--482, 2021.

\bibitem[Pang et~al.(2022)Pang, Parrish, Joshi, Nangia, Phang, Chen, Padmakumar, Ma, Thompson, He, and Bowman]{pang2022quality}
Richard~Yuanzhe Pang, Alicia Parrish, Nitish Joshi, Nikita Nangia, Jason Phang, Angelica Chen, Vishakh Padmakumar, Johnny Ma, Jana Thompson, He~He, and Samuel~R. Bowman.
\newblock {QuALITY}: Question answering with long input texts, yes!
\newblock In \emph{Conference of the North American Chapter of the Association for Computational Linguistics: Human Language Technologies}, pages 5336--5358, 2022.

\bibitem[Angelopoulos et~al.(2021)Angelopoulos, Bates, Cand{\`e}s, Jordan, and Lei]{angelopoulos2021learn}
Anastasios~N Angelopoulos, Stephen Bates, Emmanuel~J Cand{\`e}s, Michael~I Jordan, and Lihua Lei.
\newblock Learn then test: Calibrating predictive algorithms to achieve risk control.
\newblock \emph{arXiv preprint arXiv:2110.01052}, 2021.

\bibitem[Laufer-Goldshtein et~al.(2022)Laufer-Goldshtein, Fisch, Barzilay, and Jaakkola]{laufer2022efficiently}
Bracha Laufer-Goldshtein, Adam Fisch, Regina Barzilay, and Tommi~S Jaakkola.
\newblock Efficiently controlling multiple risks with pareto testing.
\newblock In \emph{International Conference on Learning Representations}, 2022.

\bibitem[Hartvigsen et~al.(2019)Hartvigsen, Sen, Kong, and Rundensteiner]{hartvigsen2019adaptive}
Thomas Hartvigsen, Cansu Sen, Xiangnan Kong, and Elke Rundensteiner.
\newblock Adaptive-halting policy network for early classification.
\newblock In \emph{Proceedings of the 25th ACM SIGKDD International Conference on Knowledge Discovery \& Data Mining}, pages 101--110, 2019.

\bibitem[Gupta et~al.(2020)Gupta, Gupta, Biswas, and Dutta]{gupta2020approaches}
Ashish Gupta, Hari~Prabhat Gupta, Bhaskar Biswas, and Tanima Dutta.
\newblock Approaches and applications of early classification of time series: A review.
\newblock \emph{IEEE Transactions on Artificial Intelligence}, 1\penalty0 (1):\penalty0 47--61, 2020.

\bibitem[Ghodrati et~al.(2021)Ghodrati, Bejnordi, and Habibian]{ghodrati2021frameexit}
Amir Ghodrati, Babak~Ehteshami Bejnordi, and Amirhossein Habibian.
\newblock Frameexit: Conditional early exiting for efficient video recognition.
\newblock In \emph{Proceedings of the IEEE/CVF Conference on Computer Vision and Pattern Recognition}, pages 15608--15618, 2021.

\bibitem[Sabet et~al.(2021)Sabet, Hare, Al-Hashimi, and Merrett]{sabet2021temporal}
Amin Sabet, Jonathon Hare, Bashir Al-Hashimi, and Geoff~V Merrett.
\newblock Temporal early exits for efficient video object detection.
\newblock \emph{arXiv preprint arXiv:2106.11208}, 2021.

\bibitem[Tang et~al.(2022)Tang, Kumar, Xin, Vyas, Li, Yang, Mao, Murray, and Lin]{tang2022temporal}
Raphael Tang, Karun Kumar, Ji~Xin, Piyush Vyas, Wenyan Li, Gefei Yang, Yajie Mao, Craig Murray, and Jimmy Lin.
\newblock Temporal early exiting for streaming speech commands recognition.
\newblock In \emph{International Conference on Acoustics, Speech and Signal Processing}, pages 7567--7571. IEEE, 2022.

\bibitem[Hartvigsen et~al.(2022)Hartvigsen, Gerych, Thadajarassiri, Kong, and Rundensteiner]{hartvigsen2022stop}
Thomas Hartvigsen, Walter Gerych, Jidapa Thadajarassiri, Xiangnan Kong, and Elke Rundensteiner.
\newblock Stop\&hop: Early classification of irregular time series.
\newblock In \emph{Proceedings of the 31st ACM International Conference on Information \& Knowledge Management}, pages 696--705, 2022.

\bibitem[Chen et~al.(2022)Chen, Zhang, Tian, Hou, Ma, and Zhou]{chen2022decoupled}
Huiling Chen, Ye~Zhang, Aosheng Tian, Yi~Hou, Chao Ma, and Shilin Zhou.
\newblock Decoupled early time series classification using varied-length feature augmentation and gradient projection technique.
\newblock \emph{Entropy}, 24\penalty0 (10):\penalty0 1477, 2022.

\bibitem[Shekhar et~al.(2023)Shekhar, Eswaran, Hooi, Elmer, Faloutsos, and Akoglu]{shekhar2023benefit}
Shubhranshu Shekhar, Dhivya Eswaran, Bryan Hooi, Jonathan Elmer, Christos Faloutsos, and Leman Akoglu.
\newblock Benefit-aware early prediction of health outcomes on multivariate {EEG} time series.
\newblock \emph{Journal of Biomedical Informatics}, 139:\penalty0 104296, 2023.

\bibitem[Vovk et~al.(2005)Vovk, Gammerman, and Shafer]{vovk2005algorithmic}
Vladimir Vovk, Alexander Gammerman, and Glenn Shafer.
\newblock \emph{Algorithmic learning in a random world}, volume~29.
\newblock Springer, 2005.

\bibitem[Papadopoulos and Haralambous(2011)]{papadopoulos2011reliable}
Harris Papadopoulos and Haris Haralambous.
\newblock Reliable prediction intervals with regression neural networks.
\newblock \emph{Neural Networks}, 24\penalty0 (8):\penalty0 842--851, 2011.

\bibitem[Lei et~al.(2018)Lei, G’Sell, Rinaldo, Tibshirani, and Wasserman]{lei2018distribution}
Jing Lei, Max G’Sell, Alessandro Rinaldo, Ryan~J Tibshirani, and Larry Wasserman.
\newblock Distribution-free predictive inference for regression.
\newblock \emph{Journal of the American Statistical Association}, 113\penalty0 (523):\penalty0 1094--1111, 2018.

\bibitem[Tibshirani et~al.(2019)Tibshirani, Foygel~Barber, Candes, and Ramdas]{tibshirani2019conformal}
Ryan~J Tibshirani, Rina Foygel~Barber, Emmanuel Candes, and Aaditya Ramdas.
\newblock Conformal prediction under covariate shift.
\newblock \emph{Advances in neural information processing systems}, 32, 2019.

\bibitem[Romano et~al.(2020)Romano, Sesia, and Candes]{romano2020classification}
Yaniv Romano, Matteo Sesia, and Emmanuel Candes.
\newblock Classification with valid and adaptive coverage.
\newblock \emph{Advances in Neural Information Processing Systems}, 33:\penalty0 3581--3591, 2020.

\bibitem[Bates et~al.(2021)Bates, Angelopoulos, Lei, Malik, and Jordan]{bates2021distribution}
Stephen Bates, Anastasios Angelopoulos, Lihua Lei, Jitendra Malik, and Michael Jordan.
\newblock Distribution-free, risk-controlling prediction sets.
\newblock \emph{Journal of the ACM (JACM)}, 68\penalty0 (6):\penalty0 1--34, 2021.

\bibitem[Angelopoulos and Bates(2023)]{angelopoulos2021gentle}
Anastasios~N. Angelopoulos and Stephen Bates.
\newblock Conformal prediction: A gentle introduction.
\newblock \emph{Foundations and Trends® in Machine Learning}, 16\penalty0 (4):\penalty0 494--591, 2023.
\newblock ISSN 1935-8237.

\bibitem[Gibbs and Candes(2021)]{gibbs2021adaptive}
Isaac Gibbs and Emmanuel Candes.
\newblock Adaptive conformal inference under distribution shift.
\newblock In \emph{Advances in Neural Information Processing Systems}, 2021.

\bibitem[Lin et~al.(2022)Lin, Trivedi, and Sun]{lin2022conformal}
Zhen Lin, Shubhendu Trivedi, and Jimeng Sun.
\newblock Conformal prediction intervals with temporal dependence.
\newblock \emph{Transactions on Machine Learning Research}, 2022.
\newblock ISSN 2835-8856.

\bibitem[Angelopoulos et~al.(2022)Angelopoulos, Bates, Fisch, Lei, and Schuster]{angelopoulos2022conformal}
Anastasios~N Angelopoulos, Stephen Bates, Adam Fisch, Lihua Lei, and Tal Schuster.
\newblock Conformal risk control.
\newblock \emph{arXiv preprint arXiv:2208.02814}, 2022.

\bibitem[Fisch et~al.(2022)Fisch, Jaakkola, and Barzilay]{fisch2022calibrated}
Adam Fisch, Tommi~S. Jaakkola, and Regina Barzilay.
\newblock Calibrated selective classification.
\newblock \emph{Transactions on Machine Learning Research}, 2022.
\newblock ISSN 2835-8856.

\bibitem[Feldman et~al.(2023)Feldman, Ringel, Bates, and Romano]{feldman2023achieving}
Shai Feldman, Liran Ringel, Stephen Bates, and Yaniv Romano.
\newblock Achieving risk control in online learning settings.
\newblock \emph{Transactions on Machine Learning Research}, 2023.
\newblock ISSN 2835-8856.

\bibitem[Lee et~al.(2023)Lee, Huang, Hassani, and Dobriban]{lee2023t}
Donghwan Lee, Xinmeng Huang, Hamed Hassani, and Edgar Dobriban.
\newblock T-cal: An optimal test for the calibration of predictive models.
\newblock \emph{Journal of Machine Learning Research}, 24\penalty0 (335):\penalty0 1--72, 2023.

\bibitem[Cauchois et~al.(2023)Cauchois, Gupta, Ali, and Duchi]{cauchois2023robust}
Maxime Cauchois, Suyash Gupta, Alnur Ali, and John~C Duchi.
\newblock Robust validation: Confident predictions even when distributions shift.
\newblock \emph{Journal of the American Statistical Association}, pages 1--22, 2023.

\bibitem[Barber et~al.(2023)Barber, Candes, Ramdas, and Tibshirani]{barber2023conformal}
Rina~Foygel Barber, Emmanuel~J Candes, Aaditya Ramdas, and Ryan~J Tibshirani.
\newblock Conformal prediction beyond exchangeability.
\newblock \emph{The Annals of Statistics}, 51\penalty0 (2):\penalty0 816--845, 2023.

\bibitem[Laufer-Goldshtein et~al.(2023)Laufer-Goldshtein, Fisch, Barzilay, and Jaakkola]{laufer2023risk}
Bracha Laufer-Goldshtein, Adam Fisch, Regina Barzilay, and Tommi Jaakkola.
\newblock Risk-controlling model selection via guided bayesian optimization.
\newblock \emph{arXiv preprint arXiv:2312.01692}, 2023.

\bibitem[Schuster et~al.(2021)Schuster, Fisch, Jaakkola, and Barzilay]{schuster2021consistent}
Tal Schuster, Adam Fisch, Tommi Jaakkola, and Regina Barzilay.
\newblock Consistent accelerated inference via confident adaptive transformers.
\newblock In \emph{Conference on Empirical Methods in Natural Language Processing}, pages 4962--4979, 2021.

\bibitem[Schuster et~al.(2022)Schuster, Fisch, Gupta, Dehghani, Bahri, Tran, Tay, and Metzler]{schuster2022confident}
Tal Schuster, Adam Fisch, Jai Gupta, Mostafa Dehghani, Dara Bahri, Vinh Tran, Yi~Tay, and Donald Metzler.
\newblock Confident adaptive language modeling.
\newblock \emph{Advances in Neural Information Processing Systems}, 35:\penalty0 17456--17472, 2022.

\bibitem[Brown et~al.(2001)Brown, Cai, and DasGupta]{brown2001interval}
Lawrence~D Brown, T~Tony Cai, and Anirban DasGupta.
\newblock Interval estimation for a binomial proportion.
\newblock \emph{Statistical science}, 16\penalty0 (2):\penalty0 101--133, 2001.

\bibitem[Miller(2012)]{miller2012simultaneous}
R.G.J. Miller.
\newblock \emph{Simultaneous Statistical Inference}.
\newblock Springer Series in Statistics. Springer New York, 2012.

\bibitem[Bauer(1991)]{bauer1991multiple}
Peter Bauer.
\newblock Multiple testing in clinical trials.
\newblock \emph{Statistics in medicine}, 10\penalty0 (6):\penalty0 871--890, 1991.

\bibitem[Ienco(2017)]{tiselac2017}
Dino Ienco.
\newblock Tiselac: time series land cover classification challenge, 2017.
\newblock \url{https://www.timeseriesclassification.com/description.php?Dataset=Tiselac}.

\bibitem[Chen et~al.(2015)Chen, Keogh, Hu, Begum, Bagnall, Mueen, and Batista]{UCRArchive}
Yanping Chen, Eamonn Keogh, Bing Hu, Nurjahan Begum, Anthony Bagnall, Abdullah Mueen, and Gustavo Batista.
\newblock The {UCR} time series classification archive, July 2015.
\newblock \url{www.cs.ucr.edu/~eamonn/time_series_data/}.

\bibitem[Alpaydin and Alimoglu(1998)]{misc_pen-based_recognition_of_handwritten_digits_81}
E.~Alpaydin and Fevzi. Alimoglu.
\newblock Pen-based recognition of handwritten digits.
\newblock UCI Machine Learning Repository, 1998.

\bibitem[Tan et~al.(2017)Tan, Webb, and Petitjean]{tan2017indexing}
Chang~Wei Tan, Geoffrey~I Webb, and Fran{\c{c}}ois Petitjean.
\newblock Indexing and classifying gigabytes of time series under time warping.
\newblock In \emph{SIAM international conference on data mining}, pages 282--290. SIAM, 2017.

\bibitem[Reyes-Ortiz et~al.(2012)Reyes-Ortiz, Anguita, Ghio, Oneto, and Parra]{misc_human_activity_recognition_using_smartphones_240}
Jorge Reyes-Ortiz, Davide Anguita, Alessandro Ghio, Luca Oneto, and Xavier Parra.
\newblock Human activity recognition using smartphones.
\newblock UCI Machine Learning Repository, 2012.

\bibitem[Hochreiter and Schmidhuber(1997)]{hochreiter1997long}
Sepp Hochreiter and J{\"u}rgen Schmidhuber.
\newblock Long short-term memory.
\newblock \emph{Neural computation}, 9\penalty0 (8):\penalty0 1735--1780, 1997.

\bibitem[Zheng et~al.(2023)Zheng, Chiang, Sheng, Zhuang, Wu, Zhuang, Lin, Li, Li, Xing, Zhang, Gonzalez, and Stoica]{zheng2023judging}
Lianmin Zheng, Wei-Lin Chiang, Ying Sheng, Siyuan Zhuang, Zhanghao Wu, Yonghao Zhuang, Zi~Lin, Zhuohan Li, Dacheng Li, Eric~P. Xing, Hao Zhang, Joseph~E. Gonzalez, and Ion Stoica.
\newblock Judging {LLM-as-a-judge with MT-Bench and Chatbot Arena}.
\newblock \emph{arXiv preprint arXiv:2306.05685}, 2023.

\bibitem[Kwon et~al.(2023)Kwon, Li, Zhuang, Sheng, Zheng, Yu, Gonzalez, Zhang, and Stoica]{kwon2023efficient}
Woosuk Kwon, Zhuohan Li, Siyuan Zhuang, Ying Sheng, Lianmin Zheng, Cody~Hao Yu, Joseph~E. Gonzalez, Hao Zhang, and Ion Stoica.
\newblock Efficient memory management for large language model serving with pagedattention.
\newblock In \emph{Proceedings of the ACM SIGOPS 29th Symposium on Operating Systems Principles}, 2023.

\bibitem[Ye et~al.(2023)Ye, Han, Liu, Angelopoulos, Griffith, Monakhova, and You]{ye2023learned}
Cassandra~Tong Ye, Jiashu Han, Kunzan Liu, Anastasios Angelopoulos, Linda Griffith, Kristina Monakhova, and Sixian You.
\newblock Learned, uncertainty-driven adaptive acquisition for photon-efficient multiphoton microscopy.
\newblock \emph{arXiv preprint arXiv:2310.16102}, 2023.

\bibitem[Kingma and Ba(2014)]{kingma2014adam}
Diederik~P Kingma and Jimmy Ba.
\newblock Adam: A method for stochastic optimization.
\newblock \emph{arXiv preprint arXiv:1412.6980}, 2014.

\end{thebibliography}
\bibliographystyle{unsrtnat}

%%%%%%%%%%%%%%%%%%%%%%%%%%%%%%%%%%%%%%%%%%%%%%%%%%%%%%%%%%%%%%%%%%%%%%%%%%%%%%%
%%%%%%%%%%%%%%%%%%%%%%%%%%%%%%%%%%%%%%%%%%%%%%%%%%%%%%%%%%%%%%%%%%%%%%%%%%%%%%%
% APPENDIX
%%%%%%%%%%%%%%%%%%%%%%%%%%%%%%%%%%%%%%%%%%%%%%%%%%%%%%%%%%%%%%%%%%%%%%%%%%%%%%%
%%%%%%%%%%%%%%%%%%%%%%%%%%%%%%%%%%%%%%%%%%%%%%%%%%%%%%%%%%%%%%%%%%%%%%%%%%%%%%%
\newpage
\appendix
\onecolumn
\renewcommand{\thefigure}{\Alph{section}.\arabic{figure}}
\renewcommand{\thetable}{\Alph{section}.\arabic{table}}
\renewcommand{\thealgorithm}{\Alph{section}.\arabic{algorithm}}
\section{Marginal Risk Control Algorithm}
\label{app:marginal_algo}
Algorithm~\ref{alg:marginal_risk_control} presents the marginal method described in Section~\ref{sec:marginal_risk_control}.
\begin{algorithm}[ht]
\caption{Fixed sequence testing for marginal risk control}
\label{alg:marginal_risk_control}
\begin{algorithmic}[1]
\STATE \textbf{Input:} Calibration set $\mathcal{D}_{\text{cal}} = \{(X_i,Y_i)\}_{i=1}^{n_{\text{cal}}}$, tolerable accuracy gap $\alpha$, significance level $\delta$, grid resolution $\Delta$.
\STATE $\hat{\lambda} \leftarrow \infty$ \hspace{0.1cm}// Use $\infty$ as a fallback if the first null is not rejected.
\STATE $\lambda \leftarrow 1$  \hspace{0.1cm}// Start testing with the largest $\lambda \in \Lambda$.
\WHILE{$\lambda \geq 0$}
    \STATE $\hat{R}_{\text{gap}} \leftarrow \frac{1}{n_{\text{cal}}}\sum_{i=1}^{n_{\text{cal}}} L_{\text{gap}}(Y_i, \hat{Y}^{\text{full}}_i, \hat{Y}^{\text{early}}_i(\tau_\lambda))$ \hspace{0.1cm} // Compute the empircal risk.
    \STATE $\hat{p} \leftarrow \text{CDF}_{\text{bin}}\left(\hat{R}_{\text{gap}} \cdot n_{\text{cal}}; n_{\text{cal}}, \alpha\right)$ \hspace{0.1cm} // Compute a p-value.
    \IF{$\hat{p} > \delta$}
    \STATE \textbf{break}  \hspace{0.1cm} // Failed to reject the null, stop testing.
    \ENDIF
    \STATE $\hat{\lambda} \leftarrow \lambda$  \hspace{0.1cm} // $H_{0,\lambda}$ was rejected, update the chosen \( \hat{\lambda} \).
    \STATE $\lambda \leftarrow \lambda - \Delta$  \hspace{0.1cm} // Next test will test a lower threshold.
\ENDWHILE
\STATE \textbf{Output:} $\hat{\lambda}$
\end{algorithmic}
\end{algorithm}

\section{Proofs}
\label{app:proofs}

\begin{proof}[Proof of Proposition~\ref{prop:marginal_risk_control}]
The validity of the proposition is a direct consequence of using fixed sequence testing.
For completeness, we add a proof that fixed sequence testing controls the FWER at level $\delta$.
Denote by $H_{0,{j}}$ the $j$-th ordered hypothesis.
If all the hypotheses are false, we trivially get that $\mathbb{P}(V \geq 1) = 0$.
Next, denote the index of the first true null by $j_0$, i.e., $H_{0,{j_0}}$ is true and the preceding $H_{0,{j'}}, \ j'<j_0$ are false. By the construction of the fixed sequence testing procedure, we may encounter this first true null only at step $j_0$ of the procedure. Now, observe that $\mathbb{P}(V \geq 1) = 1- \mathbb{P}(V = 0) = 1 - \mathbb{P}(\hat{p}_{\lambda_{j_0}}>\delta) = \mathbb{P}(\hat{p}_{\lambda_{j_0}}\leq\delta)\leq \delta$. Above, the second equality holds since the testing procedure stops the first time that any p-value exceeds $\delta$, and thus we get $V = 0$ if and only if    $\hat{p}_{\lambda_{j_0}}>\delta$; under this event, the procedure would terminate without rejecting $H_{0,\lambda_{j_0}}$ and $H_{0,\lambda_{j'}}, \ j'>j_0$. The last inequality follows from the validity of the p-value under the null \eqref{eq:p_value_property}.

\end{proof}

\begin{proof}[Proof of Proposition~\ref{prop:conditional_risk_control}]
To prove the result, it suffices to show that Algorithm~\ref{alg:valid_stage} controls the FWER at level $\delta$. First, observe that the outer loop in Algorithm~\ref{alg:valid_stage} tests the hypotheses $H_{0,\underline{\lambda}^t}$ sequentially, starting from $\underline{\lambda}^{t_{\text{max}}}$ down to $\underline{\lambda}^{1}$. As such, it follows the protocol of fixed sequence testing for FWER control. 
Second, each of the p-values $\hat{p}_{\underline{\lambda}^t}^{t'}$, corresponding to the finer null hypotheses \eqref{eq:conditional_null_t}, are valid since they are calculated using i.i.d. samples from the distribution $P_{XY \mid \hat{\tau}(X)\leq t'}$.
Third, the max p-value $\hat{p}_{\underline{\lambda}^t} = \max\{ \hat{p}_{\underline{\lambda}^t}^{t'} : t' \geq t\}$ used to test each of $H_{0,\underline{\lambda}^t}$ satisfies $\mathbb{P} \left ( \hat{p}_{\underline{\lambda}^t} \leq \delta \mid H_{0,\underline{\lambda}^t} \right ) \leq \delta$ \citep[Proposition 6]{angelopoulos2021learn}.
Combining these three arguments completes the proof.
\end{proof}

\section{Further Details on Experiments with Structured Datasets}
\label{app:structure_datasets}
In Section~\ref{sec:application_to_structured_data} of the main manuscript, we introduce the structured datasets on which we applied our methods. Here, we provide more details on each dataset, elaborate on the model architecture and training strategy, and present additional results.

\subsection{Datasets}
\label{sec:structure_datasets_details}

Table~\ref{tab:structured_datasets} provides more details on each dataset.

\begin{table}[h!]
\caption{Summary of structured datasets.}
\label{tab:structured_datasets}
\centering
\begin{tabular}{l c c c c l} 
\toprule
Dataset & \#Samples & \#Timesteps & \#Features & \#Classes & Type \\ [0.5ex]
\midrule
    \texttt{Tiselac}                & 99687 & 23    & 10    & 9     & Image    \\
    \texttt{ElectricDevices}        & 16637 & 96    & 1     & 7     & Device   \\
    \texttt{PenDigits}              & 10992 & 8     & 2     & 10    & Motion   \\
    \texttt{Crop}                   & 24000 & 46    & 1     & 24    & Image    \\
    \texttt{WalkingSittingStanding} & 10299 & 206   & 3     & 6     & Motion    \\ [1ex] 
\bottomrule
\end{tabular}
\end{table}

\subsection{Model Architecture and Training Strategy}
\label{app:structure_datasets_model}

We used a standard LSTM for feature extraction with one recurrent layer with a hidden size of 32, except for \texttt{WalkingSittingStanding} where we used 2 recurrent layers, each with a hidden size of 256. The output of the last recurrent layer is plugged to two fully connected classification heads, one for classifying the label $\hat{f}(X^{\leq t}) \in [0,1]^K$ and the other for estimating the confidence in the classification ${\hat{\pi}}(X^{\leq t}) \in [0,1]$. 
The loss $L_{\text{CE}}^{\hat{f}}$ for updating $\hat{f}$ is the cross-entropy, and the loss $L_{\text{BCE}}^{{\hat{\pi}}}$ for updating $\hat{\pi}$ is the binary cross-entropy. The whole network is trained to minimize $L_{\text{CE}}^{{\hat{f}}}(\hat{f}(X^{\leq t}),Y) + \gamma \cdot L_{\text{BCE}}^{{\hat{\pi}}}(\hat{\pi}(X^{\leq t}),B(X^{\leq t}))$, where the function $B \in \{0,1\}$ returns the value $1$ if $\hat{f}(X^{\leq t})$ correctly predicts the label $Y$ and zero otherwise. We set the hyperparameter $\gamma$ to $0.2$ in all experiments. We augment the data by fitting the model on all possible prefixes $X^{\leq t}$, $t=1,\ldots,t_{\text{max}}$. The optimizer used to minimize the objective function is Adam \cite{kingma2014adam}, with a learning rate of $0.001$, and a batch size of 64.
We allocate $1/8$ of the training samples to a validation set and optimize the model on the remaining $7/8$ of the samples. Training continues until there is no improvement in the loss on the validation set for 30 epochs. The model with the best validation set loss is then saved.

\subsection{Additional Results}
\label{sec:structure_datasets_results}

Table~\ref{tab:results} from the main manuscript summarizes the performance of the marginal and conditional methods on the structured datasets. In addition, Figure~\ref{fig:structured_datasets_results} presents more detailed results, illustrating the accumulated accuracy gap and accumulated stopping times as a function of $t$ obtained by the marginal and conditional methods.

\begin{figure}[ht]
        \centering
        {%
        \subfloat[\texttt{Tiselac}]{%
        \includegraphics[width=0.47\textwidth]{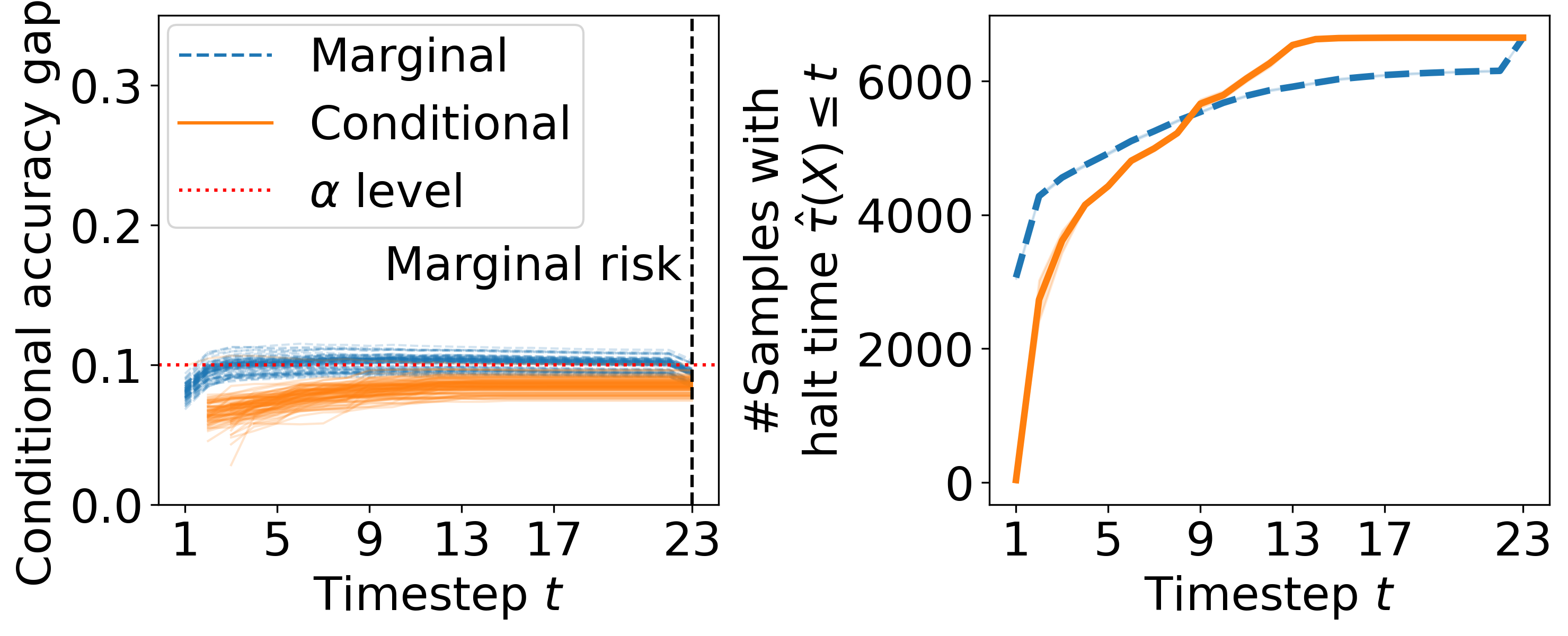}%
        }
        }%
        \hspace{20pt}
        {%
        \subfloat[\texttt{ElectricDevices}]{%
        \includegraphics[width=0.47\textwidth]{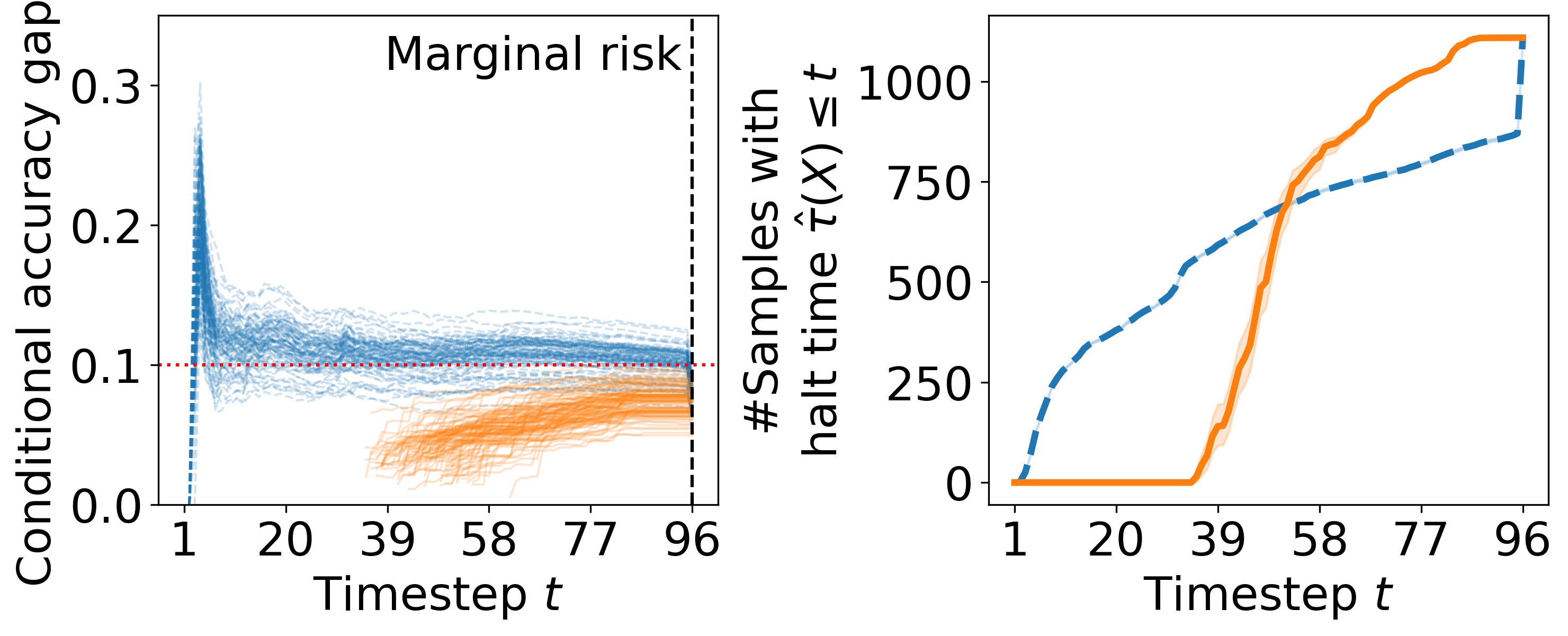}%
        }
        }%
        \hfill
        {%
        \subfloat[\texttt{PenDigits}]{%
        \includegraphics[width=0.47\textwidth]{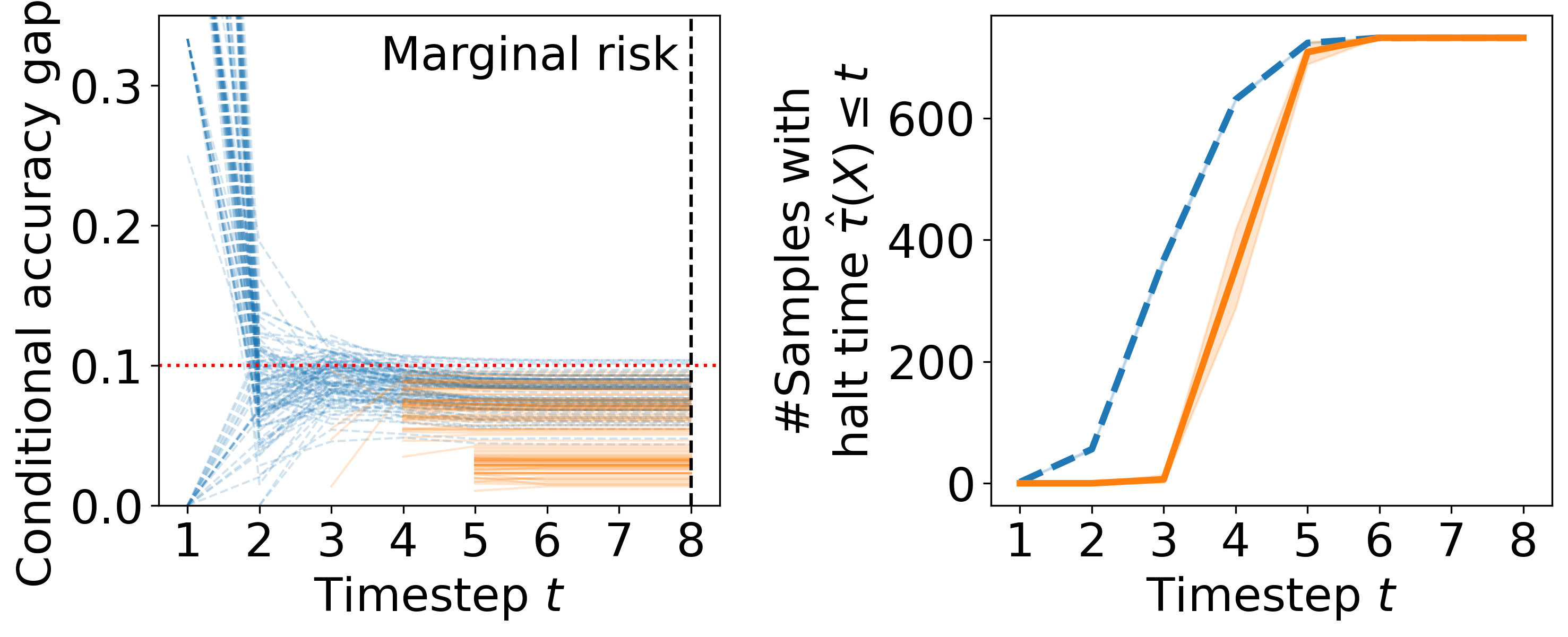}%
        }
        }%
        \hspace{20pt}
        {%
        \subfloat[\texttt{Crop}]{%
        \includegraphics[width=0.47\textwidth]{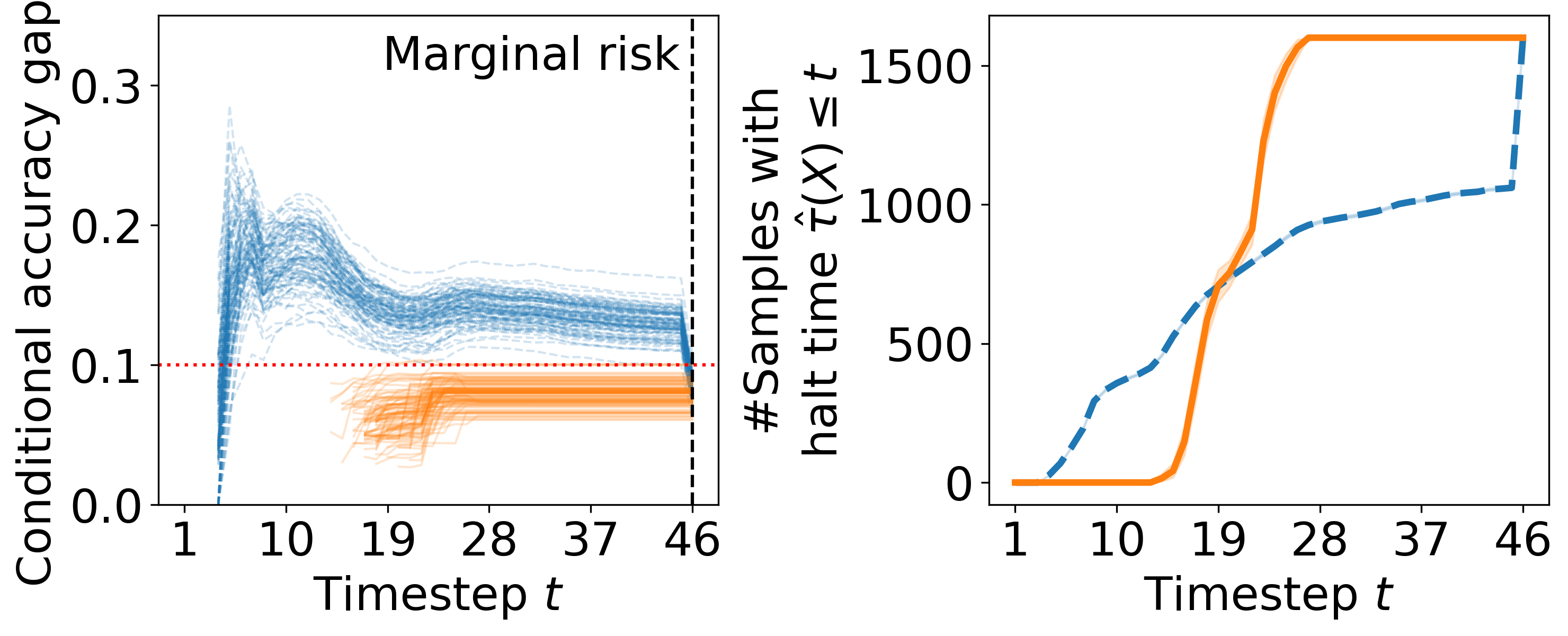}%
        }
        }%
        \hfill
        {%
        \subfloat[\texttt{WalkingSittingStanding}]{%
        \includegraphics[width=0.47\textwidth]{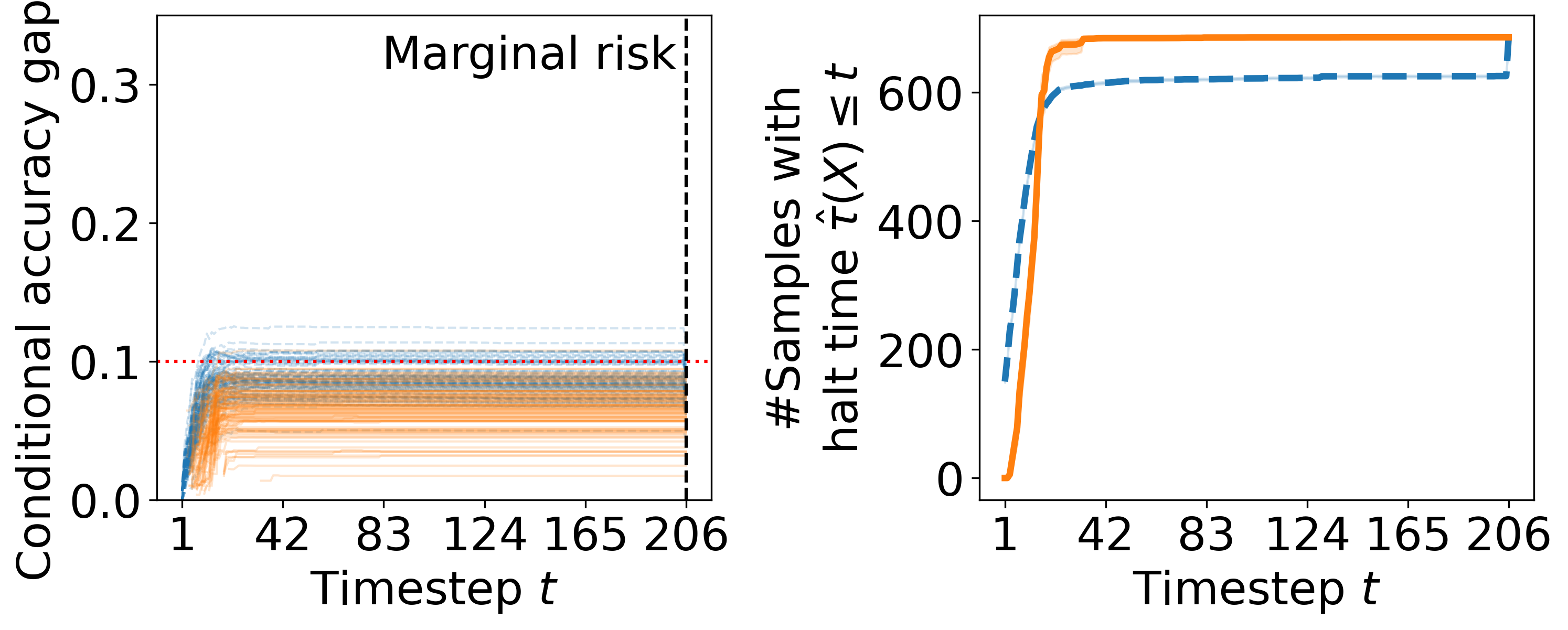}%
        }
        }%
        \caption{\textbf{Comparison between the marginal and conditional methods for the structured datasets.} The other details are as in Figure~\ref{fig:quality_conditional_vs_marginal}.}
        \label{fig:structured_datasets_results}
\end{figure}

Figure~\ref{fig:t_avg_vs_alpha} shows how different error levels $\alpha$ affect the average halt times with the conditional method.
As expected, when allowing for a higher level of risk, the calibration manages to identify thresholds that result in shorter halt times.

\begin{figure}[ht]
\centering
        \includegraphics[width=0.5\textwidth]{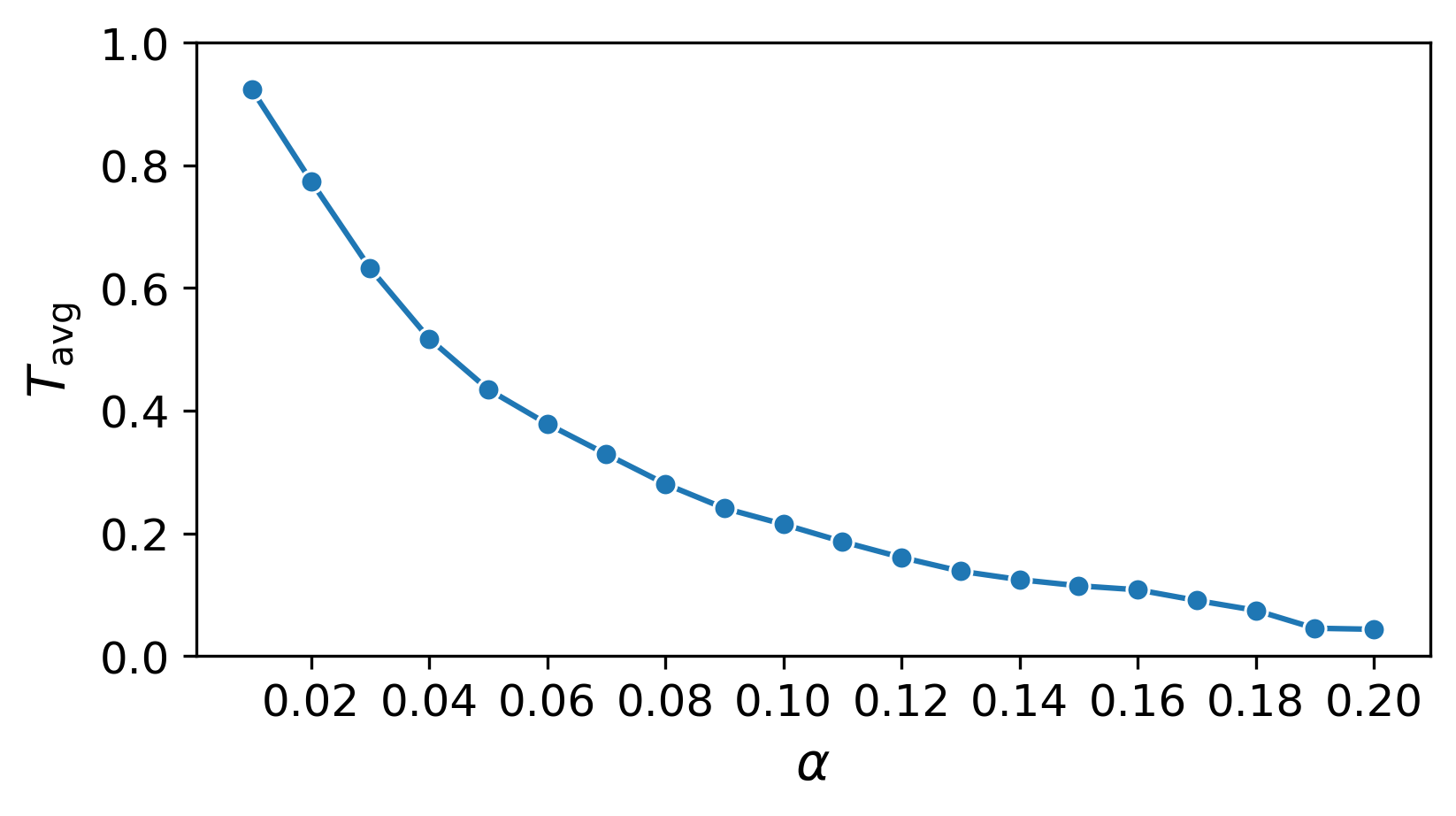}
        \caption{\textbf{Normalized halt time $T_{\text{avg}}$ vs. tolerable accuracy gap $\alpha$}. The results are averaged over 100 random splits of the \texttt{Tiselac} dataset, with (tiny) standard error bars.}
        \label{fig:t_avg_vs_alpha}
\end{figure}

\section{Ablation Study on the NLP Application}
In this section, we present an ablation study to assess the significance of the second stage in the conditional method: the testing phase.
Figure~\ref{fig:stage1_vs_stage2} summarizes the results discussed in Section~\ref{sec:nlp_application} of the main manuscript.

\begin{figure}[ht]
\centering
    {%
    \subfloat{%
    \includegraphics[width=0.48\textwidth]{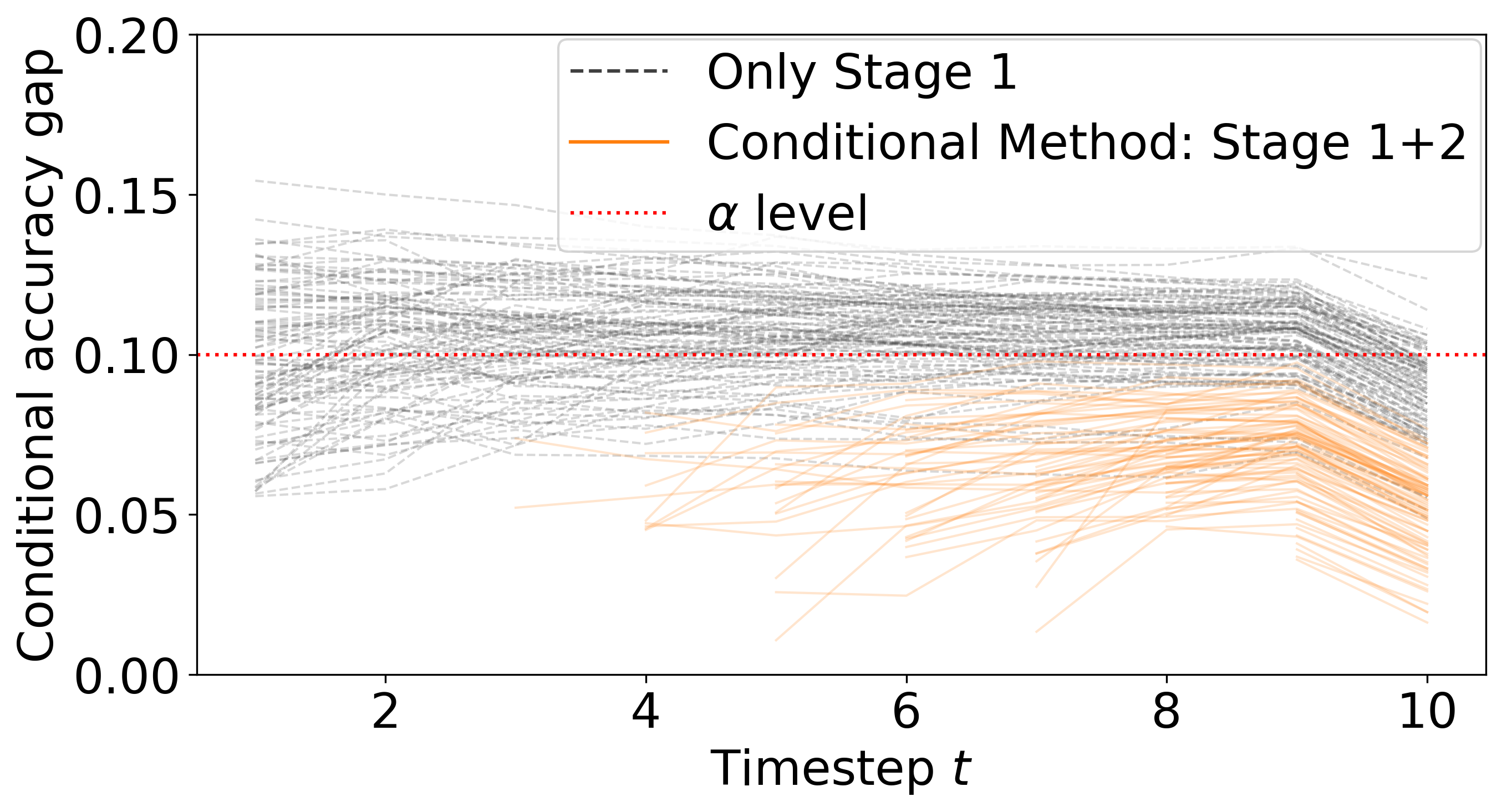}%
    \label{fig:quality_stage1_vs_stage2}%
    }
    }%
    \caption{\textbf{The importance of the testing procedure---Stage 2.} Comparison of conditional accuracy gap obtained by candidate screening (Stage 1, black curves) and by the full conditional method (Stage 1+2, orange curves). The results are presented for 100 random calibration/test splits of the \texttt{QuALITY} dataset, with each curve corresponding to a different split. }
    \label{fig:stage1_vs_stage2}
\end{figure}

%%%%%%%%%%%%%%%%%%%%%%%%%%%%%%%%%%%%%%%%%%%%%%%%%%%%%%%%%%%%%%%%%%%%%%%%%%%%%%%
%%%%%%%%%%%%%%%%%%%%%%%%%%%%%%%%%%%%%%%%%%%%%%%%%%%%%%%%%%%%%%%%%%%%%%%%%%%%%%%

\end{document}